%% file: new_main.tex
\documentclass[journal]{IEEEtran}
 
\IEEEoverridecommandlockouts
\usepackage{amssymb}
\usepackage{amsthm}
\usepackage{bm}
\usepackage{algorithmic}
\usepackage[ruled,vlined]{algorithm2e}
\usepackage{graphicx}
\usepackage{textcomp}
\usepackage{xcolor}
\usepackage{subcaption}
\usepackage{balance}

\newcommand{\cblue}{\textcolor{black}}
\newcommand{\cbluee}{\textcolor{black}}
 
\newcommand{\cred}{\textcolor{red}}

\newcommand{\appropto}{\mathrel{\vcenter{
  \offinterlineskip\halign{\hfil$##$\cr
    \propto\cr\noalign{\kern2pt}\sim\cr\noalign{\kern-2pt}}}}}

\usepackage{pgfplots}
\usepackage{siunitx}
\usepackage{mathtools}
\usepackage{filecontents}
\usepackage{pgfplots}
\usepackage{tikz}
\usetikzlibrary{patterns}
\usepackage{subcaption}
\newtheorem{thm}{Theorem}

\theoremstyle{definition}

\theoremstyle{remark}

\begin{document}

\def \btheta {\boldsymbol{\theta}}
\def \thetab {\boldsymbol{\theta}}
\def \data {\mathcal{D}}
\def \wt {\widetilde}
\def \bmu {\boldsymbol{\mu}}

% \begin{frontmatter}



\title{Bayesian data fusion with shared priors}


% \author[inst1]{Peng Wu}

% \affiliation[inst1]{organization={Electrical and Computer Engineering Department},%Department and Organization
%             addressline={Northeastern University}, 
%             city={Boston},
%             postcode={02115}, 
%             state={MA},
%             country={US}}

% \author[inst1]{Tales~Imbiriba}
% \author[inst2]{V\'ictor Elvira}
% \author[inst1]{Pau~Closas}

% \affiliation[inst2]{organization={School of Mathematics},%Department and Organization
%             addressline={University of Edinburgh}, 
%             % city={City Two},
%             % postcode={22222}, 
%             % state={State Two},
%             country={UK}}

% \thanks{This work has been partially supported by the NSF under Award ECCS-1845833.}

\author{Peng Wu$^\dagger$, Tales~Imbiriba$^\dagger$, V\'ictor Elvira$^\ast$~\IEEEmembership{Senior,~IEEE}, Pau~Closas$^\dagger$~\IEEEmembership{Senior,~IEEE.}
%\thanks{}% 
%\thanks{}%
\\
\IEEEauthorblockA{
$^\dagger$Electrical and Computer Engineering Department, Northeastern University, Boston, MA \\ \{wu.p,talesim,closas\}@northeastern.edu} \\
\IEEEauthorblockA{
$^\ast$School of Mathematics, University of Edinburgh, UK
 \\ victor.elvira@ed.ac.uk}

\thanks{This work has been partially supported by the NSF under Awards ECCS-1845833 and CCF-2326559. The work of V. E. is supported by ARL/ARO under grant W911NF-22-1-0235.}
}

\markboth{Submitted to IEEE Trans. on Signal Processing}
{Wu et al.: Bayesian data fusion with shared priors}

\maketitle

\begin{abstract}
The integration of data and knowledge from several sources is known as data fusion. When data is only available in a distributed fashion or when different sensors are used to infer a quantity of interest, data fusion becomes essential. \cblue{In Bayesian settings, a priori information of the unknown quantities is available and, possibly, present among the different distributed estimators. When the local estimates are fused, the prior knowledge used to construct several local posteriors might be overused unless the fusion node accounts for this and corrects it. In this paper, we analyze the effects of shared priors in Bayesian data fusion contexts. Depending on different common fusion rules, our analysis helps to understand the performance behavior as a function of the number of collaborative agents and as a consequence of different types of priors. The analysis is performed by using two divergences which are common in Bayesian inference, and the generality of the results allows to analyze very generic distributions.  These theoretical} results are corroborated through experiments in a variety of estimation and classification problems, including linear and nonlinear models, and federated learning schemes.
%\victor{I modified the abstract but still should be fed with the additions during the review process. It can make a difference.}
\end{abstract}

% % \includegraphics{grabs}

%Research highlights
% \begin{highlights}
% \item Theoretical analysis of distributed Bayesian data fusion, where priors are shared.
% \item Analytical results are provided for systems whose posterior is Gaussian.
% \item Experimental validation done in several relevant inference problems and connections to federated learning.
% \end{highlights}

% \begin{keyword}
% %% keywords here, in the form: keyword \sep keyword
% Data fusion \sep Bayesian inference \sep
% %% PACS codes here, in the form: \PACS code \sep code
% Distributed Machine Learning \sep Federated learning
% %% MSC codes here, in the form: \MSC code \sep code
% %% or \MSC[2008] code \sep code (2000 is the default)
% \end{keyword}

% \end{frontmatter}


\section{Introduction and related works}

% \IEEEPARstart{D}{istributed}
Distributed data fusion (DDF) is a process in which a group of agents senses their immediate environment, communicates with other agents, and aims at inferring knowledge about a specific process collectively.
Some applications include cooperative robots mapping a room \cite{dardari2015indoor}, multisensor data fusion in internet of things\cite{ding2019survey},  navigation systems \cite{dunik2020state}, medical diagnosis, or pattern recognition to name a few\cite{hall2004mathematical}. %\pau{do we have references for some/all of those applications?}

% Data Fusion intro:\cite{meng2020survey,hall1997introduction, 7515322,khaleghi2013multisensor, bakr2017distributed,ding2019survey}\pau{are the refs placed in the text? can this sentence be removed?}

There are several data fusion architectures which have been proposed. Arguably, the joint
directors of laboratories (JDL) Data Fusion Group is the most widely-used taxonomy method for data fusion-related functions \cite{meng2020survey, White1991DataFL}. It defines data fusion as a ``multilevel, multifaceted process handling the automatic detection, association, correlation, estimation, and combination of data and information from several sources.'' Two recent fusion frameworks were proposed in \cite{kokar2004formalizing,castanedo2013review}, both based on category theory and are claimed to be sufficiently general to capture all kinds of fusion techniques \cite{hall1997introduction}, including raw data fusion, feature fusion, and decision fusion.%\pau{maybe a good place to state where our contribution classifies? is it `raw'??}
%
% In the case of raw data fusion level, agent data is directly combined in situations where the local sensor data are commensurate. Alternatively, feature-level fusion involves the extraction of representative features from local sensor data, which requires a reduced information exchanged as compared to raw data fusion schemes. This is quite popular in pattern recognition, an example of feature extraction is the use of characteristics of human's face images \cite{pong2014multi}. In decision level fusion \cite{hall1997introduction}, the information is fused after each agent makes a preliminary local determination. This level of fusion methods encompasses classical inference, weighted decision methods (voting techniques), Bayesian inference and other inference methodologies. \pau{seems you are describing the fusion option from JDL, but nothing is said about the fourth one: 'fusion of relational information'} 
% In this paper, we also focus on the decision level fusion. 

Distributed data fusion research can be categorized into Bayesian or consensus-based approaches \cite{7515322}. Bayesian methods focus on preserving the full distribution of the unknowns given the data, called posterior, over the estimated process at each agent, so that sensor data can be easily and recursively merged with prior knowledge and does not need to be stored. On the other hand, consensus algorithms are designed in such a way that agents can continue to exchange information until they agree on certain parameters or quantities of interest. 
In this context, many data fusion methods focus on Bayesian methodologies, such as naive Bayesian fusion, federated Kalman filtering \cite{hashemipour1988decentralized,mahmoud2013distributed}, and other different fusion methods that may leverage data-driven models  \cite{tresp2000bayesian,bailey2012conservative, deisenroth2015distributed}. Recently, there has been increasing interest in the consensus area, particularly in the machine learning community, where a plethora of distributed learning or federated learning methods have been proposed \cite{mcmahan2021advances,Ji2021EmergingTI,wu2021personalized}.
In this paper, we focus on a Bayesian perspective to the data fusion problem. 

% Robust fusion:\cite{li2019second,bailey2012conservative,chang2010analytical}
% \pau{comment out?}

There are many data fusion challenges \cite{meng2020survey,khaleghi2013multisensor}, being the data correlation problem one of the most prominent cases. In general, the performance of DDF solutions cannot outperform a centralized scheme, which is typically considered as a benchmark. Although on the other side, DDF is intrinsically more adaptive and resilient to failures. When dealing with distributed agents, some of the challenges involve both dealing with observations impacted by the same process noise \cite{bar1981track} and also non-independence of local estimates due to \cblue{multiple} counting of data \cite{bakr2017distributed,chang2010analytical,bailey2012conservative,7515322}, which essentially means that in a distributed architecture local agent estimates may be correlated. To maintain optimality and consistency, a distributed fusion method should account for such cross-correlation issues. The \cblue{multiple-counting} problem occurs when data is utilized numerous times without the user's knowledge.\footnote{\cblue{Note that certain data fusion literature refers to this issue as the \emph{double-counting} problem. We will refer to \textit{multiple-counting} problem in this paper, since it is a more appropriate term in the considered fusion context.}}  This might be due to recirculation of data through cyclic channels or the same data traveling through several paths from another agent to the fusion node. 
To avoid \cblue{multiple counting of data}, two popular solutions are typically adopted \cite{li2019second}: arithmetic average (AA) and geometric average (GA). These two methods can be employed by most fusion methods that \textit{average} the parameters of interest across many agents. Sometimes \cbluee{the  GA solution} are referred to as Chernoff fusion approaches \cite{chang2010analytical}, also known as covariance intersection \cite{julier2017general,bakr2017distributed} under Gaussian assumptions. Related to the latter are the works on \textit{mixture of experts} \cite{nguyen2016universal}, robust \textit{product of experts} \cite{deisenroth2015distributed}, \cblue{fusion of Monte Carlo approximations \cite{saucan2023information}}, and the weighted fusion of Kalman filters \cite{xing2016multisensor,alimadadi2020object}, as well as the linear fusion of partial estimators \cite{victor201803,luengo2015bias}.

% Kalman fusion: \cite{hashemipour1988decentralized,mahmoud2013distributed,xing2016multisensor,julier2017general}

% Distributed estimators(ML,FL,SL): \cite{nguyen2016universal,lalitha2019peer,liu2021bayesian,Ji2021EmergingTI,niknam2020federated,chen2020fedbe,wu2021personalized,pham2021fusion}
% \pau{comment out?}

DDF is a field that connects to other disciplines, for that reason related works can be found also under the umbrella of model fusion, estimator fusion, distributed estimation, and distributed learning. %They mostly focus on the difference of model using and the model parameters fusion. 
Particularly, in the area of machine learning, the interest of distributed learning is growing rapidly under the so-called federated learning (FL) paradigm. Most of the research in FL focuses on frequentist approaches to inference, where data-driven neural network (NN) parameters are aggregated \cite{mcmahan2017communication,wu2021personalized,pham2021fusion} based on the AA method. Besides the fusion methodology, privacy and communication efficiency are also a key part of FL research \cite{mcmahan2021advances}. 
Recently, some works explored Bayesian FL schemes \cite{achituve2021personalized, lalitha2019peer,chen2020fedbe, liu2021bayesian,Junha2022} where the model fusion problem is addressed from a Bayesian perspective, \cblue{an approach that is thoroughly reviewed in \cite{Koliander22}. While this paper described the different fusion strategies including the CIP and CIL we discussed, the present paper provides additional theoretical analyses and connections to different application examples that help provide further practical insights.} Also within the Bayesian approach, Bayesian committee machine (BCM) \cite{tresp2000bayesian,deisenroth2015distributed,liu2018generalized,liu2020gaussian} is another framework that relates to DDF, in this case focused on Gaussian process models.

% BCM: \cite{tresp2000bayesian,deisenroth2015distributed,liu2018generalized,liu2020gaussian} \pau{comment out?}

In this paper, we analyse the use of \textit{a priori} information in distributed fusion problems, where data is not directly shared and thus kept private. Particularly, we focus on schemes where a parameter $\boldsymbol{\theta}$ is inferred locally by $M$ agents in a Bayesian setting. In this scheme, $\btheta$ takes values in an arbitrary $d$-dimensional space $\Theta$. The quantity of interest is therefore treated as a random variable for which an a priori distribution is available, $p(\boldsymbol{\theta})$. This prior is shared by all participating agents and updated using local data $\mathcal{D}_m$ to produce the a posteriori distribution of the unknown given the local information, $p(\boldsymbol{\theta} | \mathcal{D}_m)$. 
In this context, we analyse two conditional independent (exact and approximate) Bayesian fusion methodologies that employ the $M$ local posteriors and, in particular, connect those to the case where the full posterior is computed at a central node accessing all data, $p(\boldsymbol{\theta} | \mathcal{D}_1,\mathcal{D}_2,\dots,\mathcal{D}_M)$. 
% In this context, this work analyzes various Bayesian fusion methodologies that employ the $M$ local posteriors and, in particular, connects those to the case where the full posterior would be computed at a central node collecting all data, $p(\boldsymbol{\theta} | \mathcal{D}_1,\mathcal{D}_2,\dots,\mathcal{D}_M)$. 
One of the main challenges of Bayesian fusion rules is the repeated use of prior data, which is generally not properly dealt with. We discuss corrections to avoid systematic bias, as well as provide a discussion on how the number of clients affects the performance of the fusion rule when $p(\boldsymbol{\theta})$ is used in excess without corrections. 
\cblue{In practical problems of distributed Bayesian inference, not dealing properly with the multiple counting problem can yield to very sub-optimal solutions. In this paper, we discuss practical situations where this problem occurs, particularly when either the number of participating agents $M$ grows or when the prior distribution is very informative. The challenge is often that details on how prior information was used are not always available at the fusion center. Also, in other systems, efficiency is preferred over accuracy due to computational limitations, which can endanger the quality of the final posterior. An illustrative scenario is in distributed target tracking \cite{alimadadi2020object}, where multiple sensors are combined to enhance accuracy and coverage. However, due to band-limited communications the priors are not always available at the central node, in which case fusion typically occurs at the decision level without considering the shared prior influence, ultimately yielding to potentially erroneous tracking results.} 
%\cblue{In a practical scenario such as in most distributed systems, the noise will be modeled as prior information. Without considering the double counting problem, the accuracy would be compromised. Another scenario is in a tracking situation, using multiple sensors can improve tracking performance by combining their information \cite{alimadadi2020object}. However, if the information cannot be shared, fusion typically occurs at the decision level, without considering the shared prior influence, the final decision erroneous.}

% \cblue{The comments 1 reply: }

The novelty of this work is to present a rigorous analysis of distributed Bayesian fusion strategies, which advances our understanding on how prior information impacts their performance. The case of Gaussian posteriors, of interest in a Bayesian inference context due to its tractability, is investigated in detail. In particular, we provide analytic results on the impact of both the quality of the shared prior information and the number of collaborative clients. 
\cblue{These results are shown to be valid for arbitrary distributions, going beyond the Gaussian assumption.}
The theoretical findings are validated on a set of representative inference problems such as distributed linear and neural network based regression and classification, and on a federated learning scheme. This work paves the way for further methodological advances in distributed Bayesian fusion, equipped with a solid theoretical justification.

The remainder of the paper is organized as follows. Section \ref{sec:fusion} introduces the optimal (global) Bayesian solution as well as two fusion rules using local updates and assumptions. Section \ref{sec:GaussianFusion} provides insights for the important case where inferring $\boldsymbol{\theta}$ can be seen as a regression problem where the relevant distributions are approximated by a Gaussian distributions. \cblue{Section \ref{sec:convergence_analysis} provides an analysis of the two fusion rules, showing results in terms of the number of participating agents and the informativeness of the a priori distribution, under arbitrary distributions and not necessarily restricted to the Gaussian case.} Different experimental results, featuring various situations, are discussed in Section \ref{sec:experiments}, validating the theoretical results on earlier sections in relevant cases. The paper is concluded by Section \ref{sec:conclusions} with final remarks.

% In this paper, we try to present a simple and unified framework for Bayesian data fusion scheme which can apply to any related application. From simple model fusion to Kalman fusion\cite{hashemipour1988decentralized}, to which called Bayesian committee machine\cite{tresp2000bayesian,deisenroth2015distributed,liu2018generalized,liu2020gaussian}, which focus on the Gaussian process model, to Bayesian Federated learning\cite{chen2020fedbe,liu2021bayesian} and so on. 

% We also will study the properties of different fusion methods with respect to prior information and the number of the agent.  

\section{Bayesian Data Fusion schemes}\label{sec:fusion}

\begin{figure}
% \centerline{\includegraphics[width=0.45\textwidth]{figs/bayeisan_fusion_plot.png}
% }
\centering
\includegraphics[width=0.4\textwidth]{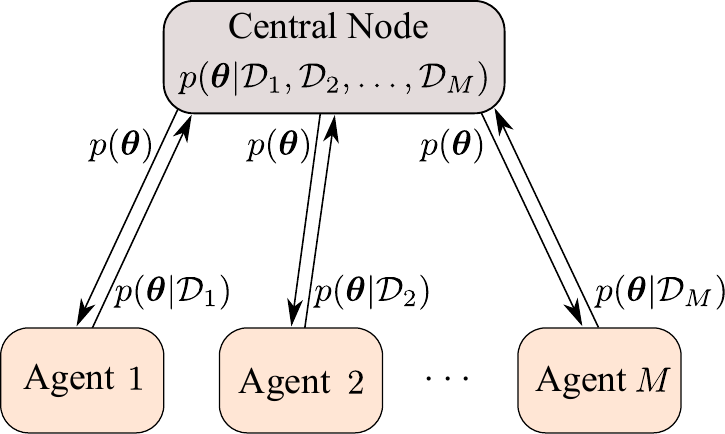}
\caption{Distributed framework of $M$ agents sharing a priori information on a quantity, $p(\btheta)$, and returning local posterior updates, $p(\boldsymbol{\theta}|\mathcal{D}_m)$. A central node is in charge of fusing the local updates with the goal of approximating the full posterior $p(\boldsymbol{\theta}|\mathcal{D}_1,\dots,\mathcal{D}_M)$.}
\label{fig:distributed}
\end{figure}

This paper considers a framework (see Fig. \ref{fig:distributed}) where observations related to an unknown quantity of interest $\boldsymbol{\theta}\in\Theta^{d}$ are obtained by $M$ distributed agents, such that the $m$-th agent has only access to its own partial data $\mathcal{D}_m$. The data is assumed to be \cblue{conditionally} independent and identically distributed (\textit{i.i.d.}) across agents. In addition, a Bayesian approach to the inference of $\boldsymbol{\theta}$ is considered, therefore a prior distribution is assumed $p(\boldsymbol{\theta})$ before data is observed. The main goal of Bayesian data fusion is to evaluate the posterior of the quantity given all the data $p(\boldsymbol{\theta}|\mathcal{D})$, however this might be unpractical due to the need to transmit local data to a central node in charge of processing it. Alternatively, distributed Bayesian data fusion provides a framework to combine the local inference results in a manner that approximates the optimal Bayesian solution.
\cblue{For the sake of clarity, Table \ref{tb:notation} provides a summary of the notation used throughout the paper.}

{\color{blue}
\begin{table}
\caption{\cblue{Table of relevant notation conventions.}}
\begin{tabular}{l|l|l}
 & \cblue{Notation} & \cblue{Definition}                                 \\ \hline
 & $\btheta$        & Parameter of interest.      \\
 & $p(\btheta)$        &    Prior distribution for $\btheta$.         \\
 & $m$, $M$  & $m$-th agent and total of $M$ agents. \\
 & $\data_m$, $\data$ & Data available to the $m$-th agent and to all agents. \\
 & $p(\btheta|\data_m)$ & $m$-th agent local posterior of $\btheta$ given $\data_m$. \\
 & $p(\btheta|\data)$     &  Centralized posterior of $\btheta$ given $\data$. \\
 & $q_0$ & Scalar value representing the variance of $p(\btheta)$.  \\
 & $\btheta_0$, $\mathbf{C}_0$ & Mean and covariance of Gaussian prior.  \\
 &  $\btheta_m$, $\mathbf{C}_m$   &  Mean and covariance of local Gaussian posterior.      \\
 & $\boldsymbol{\mu}$, $\mathbf{\Lambda}$    & Mean and covariance of  CIL's Gaussian posterior.      \\
 & $\widetilde{\boldsymbol{\mu}}$, $\widetilde{\mathbf{\Lambda}}$    &  Mean and covariance of CIP's Gaussian posterior.      \\
 \end{tabular}
 \label{tb:notation}
\end{table}
}

This section formulates the optimal Bayesian fusion strategy, as it would be implemented at a central node using all available data $\mathcal{D} = \{\mathcal{D}_1,\mathcal{D}_2,\dots,\mathcal{D}_M \}$ to compute $p(\boldsymbol{\theta}|\mathcal{D})$. Then, we establish connections to different distributed Bayesian data fusion strategies where we observe that, unless properly accounted for, the a priori information may be used multiple times. This situation can potentially cause inconsistent approximations of the optimal Bayesian solution \cbluee{\cite{bakr2017distributed,tresp2000bayesian,bailey2012conservative}}. %We propose to quantify this deviation with the Kullback-Leibler (KL) divergence between the optimal posterior and the resulting approximation under the assumptions. %, a choice that will result convenient later to provide analytical insights into this effect. 
In particular, two approximate Bayesian solutions are discussed in the remainder of this section.

% \subsection{Bayesian data fusion schemes}

\noindent\textbf{CIL fusion rule.}
% Optimal Bayesian data fusion.}
Under the assumption that datasets $\mathcal{D} = \{\mathcal{D}_1,\mathcal{D}_2,\dots,\mathcal{D}_M \}$ are \textit{i.i.d.} given $\boldsymbol{\theta}$, it is possible to express the optimal Bayesian fusion rule in terms of the local likelihoods and the globally used prior distribution. We refer to this solution \cbluee{\cite{Koliander22}} as the optimal Bayesian fusion rule under \emph{conditionally independent likelihoods} (CIL). The posterior $p(\boldsymbol{\theta}|\mathcal{D})$ is then
\begin{align}\label{eq:CILa}
p(\boldsymbol{\theta}|\mathcal{D}) 
= & \frac{p(\mathcal{D}_1,\mathcal{D}_2,\dots,\mathcal{D}_M|\boldsymbol{\theta})p(\boldsymbol{\theta})}{p(\mathcal{D}_1,\mathcal{D}_2,\dots,\mathcal{D}_M)} 
\nonumber\\
=&  \frac{p(\boldsymbol{\theta}) \prod_{m=1}^{M}p(\mathcal{D}_m|\boldsymbol{\theta})} {p(\mathcal{D}_1,\mathcal{D}_2,\dots,\mathcal{D}_M)}\;,
\end{align}
\noindent \cblue{where we assumed that the joint likelihood can be factorized into local likelihoods by the conditional independence assumption,} which can be further manipulated using Bayes' rule as
\begin{align}\label{eq:CIL}
p(\boldsymbol{\theta}|\mathcal{D}) 
% = &  \frac{p(\boldsymbol{\theta}) \prod_{m=1}^{M} \frac{p(\boldsymbol{\theta}|\mathcal{D}_m)p(\mathcal{D}_m)}{p(\boldsymbol{\theta})}}{p(\mathcal{D}_1,\mathcal{D}_2,\dots,\mathcal{D}_M)} \nonumber\\
% = & \frac{\frac{1}{p(\boldsymbol{\theta})^{M-1}}  \prod_{m=1}^{M} p(\boldsymbol{\theta}|\mathcal{D}_m)}{\int \frac{1}{p(\boldsymbol{\theta})^{M-1}}  \prod_{m=1}^{M} p(\boldsymbol{\theta}|\mathcal{D}_m) d\boldsymbol{\theta}} \nonumber \\
=  Z \frac{\prod_{m=1}^{M} p(\boldsymbol{\theta}|\mathcal{D}_m)}{p^{M-1}(\boldsymbol{\theta})} \;,
\end{align}
where 
\begin{equation}\label{eq:normalization_Z}
    Z = \frac{\prod_{m=1}^{M} p(\mathcal{D}_m)}{p(\mathcal{D}_1,\mathcal{D}_2,\dots,\mathcal{D}_M)}    \;,
\end{equation}
is the normalization constant and $p(\boldsymbol{\theta}|\mathcal{D}_m)$ represents the local posterior for the $m$-th agent.

%\pau{maybe it is worth adding a discussion on the pros/cnos of this approach: can it be implemented distributedly? what information is needed? is that $Z$ important?}

%\peng{reply:}
This approach is optimal under the \textit{i.i.d.} data assumption, which enables the likelihood factorization in \eqref{eq:CIL}. It presents a Bayesian approach for data fusion when the same prior is used by the $M$ users. This solution enables for a distributed implementation, where prior is sent to each agent who returns their local posteriors in exchange.
It is easy to see that this Bayesian fusion rule for the case where each agent uses different local prior information $p_m(\boldsymbol{\theta})$ is
\begin{equation}\label{eq:CIL_diffprior}
p(\boldsymbol{\theta}|\mathcal{D}) 
=  Z p(\boldsymbol{\theta}) \prod_{m=1}^{M} \frac{ p(\boldsymbol{\theta}|\mathcal{D}_m)}{p_m(\boldsymbol{\theta})} \;, 
\end{equation}
\noindent where the global prior $p(\boldsymbol{\theta})$ might differ from the local prior distributions.
%\noindent where $p_M(\data) = \int \prod_{m=1}^M p(\data_m|\btheta) p_m(\btheta) d\btheta$.

% The benefits of it firstly is that it shares the prior information among the distributed system which is easy to manipulate and control, and also it's optimal unlike the CIP as following discussion.
%
%The disadvantage of this method is that we may be given a bad results if provide bad prior information especially in some complex nonlinear model. Plus, if the dataset among the distributed system don't have much similarity , using the same prior may not help.

%\peng{change the title to "approximate Bayesian data fusion"?}\pau{ok}
\noindent\textbf{CIP fusion rule.}
% Approximate Bayesian data fusion.}
There are situations where distributing the priors across the collaborative agents is not possible and, instead, the data fusion node has access to the local posteriors only. These agents might have used the same priors for the quantity or different across agents. A popular fusion rule is the so-called product of experts (PoE) \cbluee{\cite{deisenroth2015distributed,nguyen2016universal}} by which the local posteriors are multiplied together in order to produce a global posterior estimate, which is necessarily an approximation of the optimal solution in \eqref{eq:CIL}. 
%If we don't consider about the full posterior, it can be approximated as equation 2, which is called Product of Expert usually, but here, we still use the same prior for all clients \peng{where the PoE don't have this limitation}. 
Following the nomenclature used in this paper, we refer to this approach \cbluee{\cite{deisenroth2015distributed,Koliander22}} as the approximate Bayesian fusion rule under \emph{conditionally independent posterior} (CIP), where the main assumption is that global posterior can be factorized into a product of the local posteriors. The resulting posterior $\wt p(\btheta|\data)$ is an approximation of the \textit{true} posterior $p(\btheta|\data)$ which is accurate when, indeed, the assumption holds. More precisely,
%
% \begin{align}\label{eq:CIP}
% p(\boldsymbol{\theta}|\mathcal{D}_1,\mathcal{D}_2,\dots,\mathcal{D}_M) \appropto \prod_{m=1}^{M} p(\boldsymbol{\theta}|\mathcal{D}_m)
% % =p(\theta_{ind}) \cred{=\widehat p(\boldsymbol{\theta}|\mathcal{D}_1,\mathcal{D}_2,\dots,\mathcal{D}_M)} 
% \end{align}
%
\begin{align}\label{eq:CIP}
% p(\boldsymbol{\theta}|\mathcal{D}_1,\mathcal{D}_2,\dots,\mathcal{D}_M) 
\wt p(\btheta|\data) = \wt Z \prod_{m=1}^{M} p(\boldsymbol{\theta}|\mathcal{D}_m)\;,
\end{align}
% \noinhere 
where $\wt Z$ is a normalization term and $p(\boldsymbol{\theta}|\mathcal{D}_m)$ is $m$-th local posterior as computed by each of the $M$ agents. %For the detail formula, see appendix. 
\cblue{Notice that the prior $p(\boldsymbol{\theta})$ is reused multiple times in CIP. In contrast, CIL rule in \eqref{eq:CIL} is optimally dealing with that issue by virtue of the dividing factor $p^{M-1}(\boldsymbol{\theta})$.}

% We will evaluate and compare the CIL with CIP in mainly two aspects, one is use the risk or the accuracy as the metric to compare these two methods, another is to use the KL divergence to measure their difference. 

% \pau{add here also the KL divergence definition in Section 2.B}

% \subsection{The KL divergence}
 
The results in Section \ref{sec:experiments} will provide a comparison between CIP and CIL approaches. In addition, we are interested in developing analytical tools to gain understanding about their differences. 
Mainly, \cblue{CIP in \eqref{eq:CIP} differs from CIL in \eqref{eq:CIL} in its underlying assumption that the global posterior can be factorized into local posteriors.} %While in Section \ref{sec:experiments} we compare their inference performances, we are also interested in evaluating how different those solutions are in general settings. To achieve the latter, we investigate the KL divergence between both optimal and approximate posterior distributions. 
We propose to quantify this deviation using \cblue{particular cases of the $f$-}divergence between CIL and CIP solutions, 
%We will use the KL divergence between the optimal Bayesian fusion rule (referred to as CIL) and the approximate rule (or CIP), 
such that we can understand theoretically when it is worth accounting for the \cblue{multiple} counting of the prior \cblue{which is discussed in the Section \ref{sec:convergence_analysis}}.

\section{Fusion of Gaussian estimators}\label{sec:GaussianFusion}

Gaussian assumptions are often considered in order to allow for inferential tractability \cite{Kim08}. Furthermore, it is also common to approximate the posterior by a Gaussian distribution, for instance in the context  the so-called Gaussian filters \cite{Ito00} or Laplace approximations (see for instance INLA \cite{lindgren2015bayesian}). 

In this section, \cblue{we analyze the relevant case where $\Theta = \mathbb{R}^d$ in a Gaussian context, such that inferring $\btheta\in\mathbb{R}^{d}$ can be interpreted as a regression problem or a classification task
% when Laplace approximation is considered 
\cite{bishop2006pattern}.}
In this context, we assume that the prior on the parameter is normally distributed, $\btheta\sim \mathcal{N}(\mathbf{\btheta}_0,\mathbf{C}_0)$, and that the distribution of the local likelihood for the $m$-th agent is 
\begin{align}\label{eq:regression_example}
%   \mathbf{y}_m = f(\btheta, \mathbf{x}_m) + \mathbf{r},
   \mathbf{y}_{n,m} | \btheta \sim \mathcal{N}(\mathbf{f}(\btheta, \mathbf{x}_{n,m}),\mathbf{R}_m)
\end{align}
% \pau{$\btheta$ should be $\btheta$}
where $\mathbf{x}_{n,m}\in\mathbb{R}^{d_x}$ is the $n$-th feature input vector, $\mathbf{f}(\btheta, \mathbf{x}_{n,m})$ is a mapping function from those inputs to observed data $\mathbf{y}_{n,m}\in\mathbb{R}^{d_y}$. $n=1,\dots , N_m$ denotes the sample index for the $m$-th agent. The local dataset, as described in Section \ref{sec:fusion}, is then composed of the $N_m$ pairs $\mathcal{D}_m = \{\mathbf{y}_{n,m} , \mathbf{x}_{n,m}\}_{n=1}^{N_m}$, and $\mathbf{R}_m$ denotes the observation covariance matrix, which potentially can be different across agents.
% the noise distributed as the Gaussian distribution $\mathcal{N}(\mathbf{0},\mathbf{R})$.
Note that when the mapping is a linear function on $\btheta$, then the posterior is also normally distributed and its parameters can be optimally computed \cite{bishop2006pattern}. In general problems, \cblue{if the function $\mathbf{f}(\cdot,
\mathbf{x}_{n,m})$ is non-linear in the parameter $\btheta$, one can still make an approximation that the posterior would be approximately Gaussian and compute its mean/covariance leveraging for instance a Laplace approximation (e.g. an approach discussed in \cite{bishop2006pattern}) to linearize the model.} %can use several methods to approximate as Gaussian, this case especially using Laplace approximation in classification problem. 

Under the Gaussian approximation, the resulting local posterior $p(\boldsymbol{\theta}|\mathcal{D}_m) \approx \mathcal{N}(\btheta_{m},\mathbf{C}_m)$ with %can then be obtained, we can derive by the Bayesian theory and get like this following equation:
%
% \begin{align}\label{eq:local_posterior}
% &\btheta_{m}=\mathbf{C}_m\left(\mathbf{C}_{0}^{-1} \btheta_{0}+ \boldsymbol{F}^{\top}\mathbf{R}^{-1} \mathbf{y}\right) \;,\\
% &\mathbf{C}_m^{-1}=\mathbf{C}_{0}^{-1}+ \boldsymbol{F}^{\top} \mathbf{R}^{-1} \boldsymbol{F}
% \end{align}
% \noindent where $\boldsymbol{F}$ is the linearization of the model.
%
\begin{align}\label{eq:local_posterior}
&\btheta_{m}=\mathbf{C}_m\left(\mathbf{C}_{0}^{-1} \btheta_{0} + \sum_{n=1}^{N_m} \boldsymbol{F}_{n,m}^{\top}\mathbf{R}_m^{-1} \mathbf{y}_{n,m}\right) \;,\\
&\mathbf{C}_m^{-1}=\mathbf{C}_{0}^{-1} + \sum_{n=1}^{N_m} \boldsymbol{F}_{n,m}^{\top} \mathbf{R}^{-1} \boldsymbol{F}_{n,m},
\end{align}
\noindent where $\boldsymbol{F}=\partial \mathbf{f}/\partial\btheta$ is the Jacobian of the model, resulting from its linearization (or the coefficients of the model in case it is already linear).

\noindent\textbf{CIL fusion rule.}
The optimal Bayesian fusion rule under conditionally independent likelihoods, or CIL for short, under (approximately) local Gaussian posteriors leads to a Gaussian posterior given by %under the assumption that the local posteriors are Gaussian distributed results in another Gaussian distribution for the fused posterior, given by  
% {\color{blue}
\begin{align}\label{eq:gassuian_Correlated fusion}
p(\btheta|\mathcal{D}_1,\dots, \mathcal{D}_M) 
% = & \frac{\frac{1}{p^{M-1}(\btheta)}  \prod_{m=1}^{M} p(\btheta|\mathcal{D}_m)}{\int \frac{1}{p^{M-1}(\btheta)}  \prod_{m=1}^{M} p(\btheta|\mathcal{D}_m) d\btheta} 
% \nonumber \\
%= & \frac{1}{(2 \pi)^{\frac{k}{2}} \operatorname{\det}(\mathbf{\Lambda}^{-1})^{\frac{1}{2}}}\exp \{ 
%    -\frac{(\btheta-\boldsymbol{\mu})^{\top}\mathbf{\Lambda}(\btheta-\mathbf{\boldsymbol{\mu}})}{2}\} \nonumber \\
= & \mathcal{N}(\mathbf{\boldsymbol{\mu}}, \mathbf{\Lambda}^{-1})\;,
\end{align}
% }
where \cblue{(see {Appendix~}\ref{ap:GaussianCIL} for details)} the posterior precision matrix can be obtained as 
\begin{align}\label{eq:variance_close}
\mathbf{\Lambda} = \sum_{m=1}^{M}\mathbf{C}_m^{-1}-\mathbf{C}_{0}^{-1}(M-1)\;,
\end{align}
\noindent which depends on the local and a priori covariance matrices. The posterior mean is then
%
% it's robust version
% \begin{align}\label{eq:variance_close}
% \mathbf{\Lambda} = \sum_{m=1}^{M}\beta_{m}\mathbf{C}_m^{-1}-\mathbf{C}_{0}^{-1}(\sum_{m=1}^{M} \beta_{m}-1)
% \end{align}
%
%where $\mathbf{C_m}$ is the variance of the weights while $\mathbf{C}_{0}$ is the prior variance. $\mathbf{\Lambda}$ is precision matrix. 
%
%and $\boldsymbol{\mu}$ as: 
\begin{equation}\label{eq:mean_CIL}
    \boldsymbol{\mu} = \mathbf{\Lambda}^{-1}\left(\sum_{m=1}^{M}\mathbf{C}_m^{-1} \btheta_m -(M-1) \mathbf{C}_{0}^{-1} \btheta_{0}\right)
    % {\sum_{m=1}^{M}\mathbf{C}_m^{-1}-\mathbf{C}_{0}^{-1}(M-1)} 
    \;,
\end{equation}
which can be rearranged as
\begin{equation}\label{eq:mean_CIL2}
    \boldsymbol{\mu} = \sum_{m=1}^{M}\mathbf{\Xi}_m \btheta_m + \mathbf{\Xi}_0 \btheta_0 = \sum_{m=0}^{M}\mathbf{\Xi}_m \btheta_m\;,
\end{equation}
\noindent with the weight matrices as %$\mathbf{\xi}_m=\frac{\mathbf{C}_m^{-1}}{\sum_{m=1}^{M}\mathbf{C}_m^{-1}-\mathbf{C}_{0}^{-1}(M-1)}$, $m\neq 0$, and $\xi_0 =  \frac{-(M-1) \mathbf{C}_{0}^{-1}}{\sum_{m=1}^{M}\mathbf{C}_m^{-1}-\mathbf{C}_{0}^{-1}(M-1)}$
% \begin{equation}
%     \mathbf{\xi}_m = \left\{ 
%     \begin{array}{cc}
%          \mathbf{C}_m^{-1}\sum_{m=1}^{M}\mathbf{C}_m^{-1}-\mathbf{C}_{0}^{-1}(M-1) & m\neq 0 \\
%          \frac{-(M-1) \mathbf{C}_{0}^{-1}}{\sum_{m=1}^{M}\mathbf{C}_m^{-1}-\mathbf{C}_{0}^{-1}(M-1)} & m=0
%     \end{array}
%     \right.
% \end{equation}
% \pau{I think the correct expression for the weights is this:}
\begin{equation}
    \mathbf{\Xi}_m = \left\{ 
    \begin{array}{cc}
         \mathbf{\Lambda}^{-1} \mathbf{C}_m^{-1}, & m\neq 0, \\
         (1-M) \mathbf{\Lambda}^{-1}  \mathbf{C}_{0}^{-1}, & m=0.
    \end{array}
    \right.
\end{equation}
% \pau{also, I noticed that these are not actually weights, but matrices...}

% robust version
%  $\mathbf{\xi}_m=\frac{\beta_{m}\mathbf{C}_m^{-1}}{\sum_{m=1}^{M}\beta_{m}\mathbf{C}_m^{-1}-\mathbf{C}_{0}^{-1}(\sum_{m=1}^{M} \beta_{m}-1)}$, $m\neq 0$, and $\xi_0 =  \frac{-(\sum\beta_{m}-1) \mathbf{C}_{0}^{-1}}{\sum_{m=1}^{M}\beta_{m}\mathbf{C}_m^{-1}-\mathbf{C}_{0}^{-1}(\sum_{m=1}^{M} \beta_{m}-1)}$

\noindent\textbf{CIP fusion rule.}
Under the Gaussian assumption, the approximate fusion rule in \eqref{eq:CIP}, or CIP in short, results in 
\begin{equation}\label{eq:gassuian_Correlated fusion_approx}
\wt p(\btheta|\mathcal{D}_1, \mathcal{D}_2,\dots,\mathcal{D}_M) 
=  \mathcal{N}(\widetilde{\mathbf{\boldsymbol{\mu}}}, \widetilde{\mathbf{\Lambda}}^{-1})\;,
\end{equation}
\noindent where 
%With CIP, which prior part will be eliminated, the $\mu$ and $\Sigma$ as:
\begin{align}
    \widetilde{\boldsymbol{\mu}} &= \widetilde{\mathbf{\Lambda}}^{-1} \sum_{m=1}^{M}\mathbf{C}_m^{-1} \btheta_m, \label{eq:mean_CIP}\\
    \widetilde{\mathbf{\Lambda}} &= \sum_{m=1}^{M}\mathbf{C}_m^{-1}\;, \label{eq:cov_CIP}
\end{align}
\noindent as detailed in \cite{victor201803} and which can also be interpreted as a particular case of the Gaussian CIL rule \eqref{eq:gassuian_Correlated fusion} when no a priori information is considered in the fusion stage. That is, intuitively, that $\mathbf{C}_0$ takes very large values such that ${\boldsymbol{\mu}} \rightarrow \widetilde{\boldsymbol{\mu}}$ and ${\mathbf{\Lambda}} \rightarrow \widetilde{\mathbf{\Lambda}}$.

\cblue{\section{Convergence analysis of CIL and CIP fusion strategies.}\label{sec:convergence_analysis}}

We are interested in understanding when the two fusion rules are equivalent. That is, when accounting for the multiple use of the a priori information makes a difference and when, on the contrary, can be neglected thus simplifying the calculus. To that aim, 
\cblue{we use two divergences that belong to the family of $f$-divergences to quantify the similarities of the posteriors resulting from both methods. Namely, we consider the Kullback-Leibler (KL) and the $\chi^2$ divergences. While KL is widely used in probability and information theory \cite{pardo2018statistical}, $\chi^2$-divergence is highly relevant in Bayesian inference and importance sampling \cite{elvira2019generalized}, since it relates to the variance of the Monte Carlo estimators that attempt to approximate moments of the posterior (see for instance \cite{agapiou2017importance, miguez2017performance,elvira2021advances}). The two divergences are introduced here and particularized for the CIL and CIP expression of interest in this work.}

The KL divergence \cblue{between CIL and CIP}, for a general model where $M$ agents are fused, results in 
\begin{align} \label{eq:KL}
&\mathrm{KL}_M\left(p(\btheta|\data)\parallel\wt p(\btheta|\data)\right) 
% \\ \nonumber
 = \int p(\btheta|\data) \log\left(\frac{p(\btheta|\data)}{\wt p(\btheta|\data)}\right)d\btheta\\ \nonumber
 =&\log\left(\frac{p_M(\data)}{p(\data)}\right)  - (M-1) \int  \frac{p(\data|\btheta) p(\btheta)}{p(\data)}    \log(p(\btheta))    d\btheta,
\end{align}
\noindent where we explicitly note the number of agents as a subindex in the KL for clarity. \cblue{We notice the dependence of the divergence with $M$ and the parameter's prior distribution (see Appendix \ref{ap:thm_M} for additional details on how this expression was obtained)}:

\cblue{The $\chi^2$ divergence between the fusion rules can be generally expressed as
\begin{align}
    \chi^{2}_M&(p(\btheta|\data)\parallel\wt p(\btheta|\data))\ \ = %\ \ \int\frac{\left(p(\btheta|\data)-\wt p(\btheta|\data)\right)^{2}}{\wt p(\btheta|\data)}d\btheta \\ \nonumber
    %&= 
    \int\frac{p^2(\btheta|\data)}{\wt p(\btheta|\data)} d\btheta - 1 \\\nonumber
    % &= \int \frac{\frac{p^2(\data|\btheta)p^2(\btheta)}{p^2(\data)}}{\frac{p(\data|\btheta)p^M(\btheta)}{p_M(\data)}} d\btheta - 1 \\\nonumber
    % & = \int p(\data|\btheta)p^{2-M}(\btheta) \frac{p_M(\data)}{p^2(\data)} d\btheta -1 \\\nonumber
    % & = \int p(\btheta|\data)p^{1-M}(\btheta) \frac{p_M(\data)}{p(\data)} d\btheta - 1 \\\nonumber
    & =  \int p(\btheta|\data)p^{M-1}(\btheta)d\btheta \int p(\btheta|\data)p^{1-M}(\btheta)d\btheta -1,
\end{align}
\noindent where we again notice the dependence on $M$ and $p(\btheta)$. Additional details regarding this expression can be consulted in Appendix \ref{ap:chi2div}.}
% \victor{Isn't all this new? It should be in blue. I would expect many more things to be in blue (it is a psychological thing for reviewers to realize a lot of rework was done).}
% \cblue{More details can be found in Appendix, Equation \eqref{eq:chi-square2}.}

% \noindent where we make explicit its dependence on the number of agents $M$ for the sake of clarity in the upcoming results.
% The divergence tends to zero as the two distributions become more similar. In that sense, we provide two results showing convergence of CIP to CIL, which are stated as the following two Theorems. 

\cblue{Based on these divergences we provide an analysis of convergence between CIL and CIP fusion rules with respect to the design parameters related to the number of agents and the informativeness of the prior on $\btheta$.
}

\cblue{Before reaching our main results we define the fusion setup as follows.
Let $\mathcal{D}$ be a global dataset of \textit{i.i.d.} observations, which are related to a parameter of interest $\btheta\in\mathbb{R}^{n_\theta}$. Let $\mathcal{D}$ be processed by $M$ local agents, each observing mutually exclusive sets $\mathcal{D} = \{\mathcal{D}_1,\mathcal{D}_2,\dots,\mathcal{D}_M \}$ and sharing the same a priori information from the quantity of interest $p(\btheta)$. The $m$-th agent updates the prior to compute its local posterior, $p(\btheta|\data_m)$, which are fused at a central node either using the CIL Bayesian fusion rule $p(\btheta|\data)$ or the CIP Bayesian fusion rule $\wt p(\btheta|\data)$.}
\cblue{Finally, let $q_0$ be a scalar defining the variance of the prior $p(\btheta)$.}

% \textcolor{red}{changed the theorem orders, as the theorem 1 should be more easier for any situation, but for theorem 2, need to mention two divergence and exponential family}

\cblue{
\begin{thm}\label{theorem_q}
For $\wt p(\btheta|\data)$ and $p(\btheta|\data)$ belonging to arbitrary distribution families, $\wt p(\btheta|\data)$ asymptotically tends to $p(\btheta|\data)$ as the shared prior becomes non-informative, that is, their divergence is such that
$$\lim_{q_0\rightarrow \infty} d (p(\btheta|\data) \parallel \wt p(\btheta|\data)) = 0 \;, \quad \forall M\geqslant 1$$
where $q_0$ is a parameter that models the a priori uncertainty. %Gaussian assumption on the likelihood $p(\data_m|\btheta)=\mathcal{N}(\bmu(\btheta),\mathbf{R})$ and prior $p(\btheta)=\mathcal{N}(\btheta_0,q_0\mathbf{I})$ distributions.
\end{thm}
\begin{proof}
Details can be found in {Appendix~}\ref{ap:thm_q} for the general case, as well as a particularization to the Gaussian posterior assumption.
\end{proof}}
% \cblue{[Thinking of this one, I would suggest to keep it and highlight that the main conclusion is that this equivalence is independent of $M$. Maybe we can use an arbitrary divergence instead of KL.]}

\cblue{
\begin{thm}\label{theorem_M}
For $\wt p(\btheta|\data)$ and $p(\btheta|\data)$ belonging to arbitrary distribution families, their divergence increases with $M$, that is,
$$d_{M+1}(p(\btheta\parallel\data)|\wt p(\btheta|\data)) \geq d_M(p(\btheta|\data) \parallel \wt p(\btheta|\data))$$
% when both local posteriors and a priori are normally distributed. 
\noindent with $d_M(\cdot\parallel\cdot)$ representing either the KL or $\chi^2$ divergences.
\end{thm}
\begin{proof}
Details in {Appendices~}\ref{ap:thm_M} and \ref{ap:chi2div} for the KL and $\chi^2$ divergences, respectively.
\end{proof}
}

\cblue{Theorem \ref{theorem_q} states that CIL converges to CIP as the a priori distribution on $\btheta$ becomes less informative. %Notice that the result assumes that $\mathbf{C}_0 = q_0\mathbf{I}$ for the sake of simplicity.
Notice that the result shown in Appendix~\ref{ap:thm_q} is general for any assumed posterior distribution, although we complementary provide a result under the Gaussian case, which is widely used in the context of Bayesian inference since, under mild conditions, the posterior converges in distribution to a normal according to the Bernstein–von Mises Theorem
 \cite{kleijn2012bernstein}. 
}

Intuitively, the result in Theorem \ref{theorem_M} shows that for a fixed budget of $N$ samples, increasing the number of agents $M$ increases the separation between the two global fusion solutions. That is, as more agents participate less data per agent is available, in which case the a priori distribution on $\btheta$ becomes more relevant and the issue of reusing the prior emerges. \cblue{This result is valid for both KL and $\chi^2$ divergences under arbitrary distribution families.}

% Secondly, the main result in Theorem \ref{theorem_q} states that CIL converges to CIP as the a priori distribution on $\btheta$ becomes less informative. Notice that the result assumes that $\mathbf{C}_0 = q_0\mathbf{I}$ for the sake of simplicity.
%
% Note that our analysis is based on a Gaussian assumption, which is widely used in the context of Bayesian inference since, under mild conditions, the posterior converges in distribution to a normal  according to the Bernstein–von Mises Theorem \cite{kleijn2012bernstein}. 

% \cblue{
% \section{Extensions to non-Gaussian distributions}}

% \cblue{
% [Generalization of Theorem 1 to the exponential distribution (with KL)]
% \\
% [Generalization of Theorem 1 to the exponential distribution (with $\chi^{2}$)]
% \\
% [Generalization of Theorem 2 to the exponential distribution (or any?) and an arbitrary divergence]
% }

\section{Experiments}\label{sec:experiments}

We validate the obtained results on various relevant inference problems. Namely, $(1)$ an estimation problem, where a set of $M$ linear regression models are employed to update the prior with local data, which are subsequently fused to produce global estimates;  $(2)$ a binary classification problem in which $M$ models are used locally to produce classification results that are then fused to produce a global result; $(3)$
another general classification problem, where $M$ neural networks are locally used to produce classification solutions that are then aggregated into a global Bayesian classifier; and $(4)$ an application in the context of federated learning, which in this case updates the prior recursively instead of just once, unlike the previous experiment. %e time fusion as above.

\subsection{Distributed estimation of linear models}\label{sec:exa}

% \subsection{Toy example of regression}
For the experiments presented in this subsection we generated a synthetic data set following a linear model, with respect to the parameters $\btheta$, embedded in noise:
% consider  synthetic dataset construceted using a linear model embeded in Gaussian noise with $\mathcal{N}(\mathbf{0},\mathbf{R})$, set $\mathbf{R} = r_m \mathbf{I}$. The equation show below:

\begin{align}\label{eq:regression_example}
\mathbf{y}_m=\btheta^\top \bm{\phi}(\mathbf{x}_m) + \mathbf{r}_m ,
\end{align}

where $\mathbf{r}_m \sim \mathcal{N}(\mathbf{0},\mathbf{R})$, with $\mathbf{R} = r \mathbf{I}$, and $\bm{\phi}:\mathbb{R}^{d_x}\to\mathbb{R}^d,\mathbf{x}\mapsto \bm{\phi}(\mathbf{x})$ is an arbitrary nonlinear function.
In our experiment, we set $r=4$, the observation dimension is $1$, and $\bm{\phi}$ is identity \cblue{function}.
% matrix.
The $\btheta$ is a real random integer vector generated uniformly between $-10$ and $20$ with $d_x=6$ dimensions. The total number of points generated is around $700$ for training and around $300$ for testing. Then, we split the dataset into several parts, each corresponding to one of the local clients' training data.
% \cblue{Describe what the experiment is...}
We estimated local parameters following equation \eqref{eq:local_posterior} and fused using the approaches as discussed in Section~\ref{sec:GaussianFusion}. Furthermore, we computed the evolution of the KL divergence, and test MSE with both prior variance $q_0\in\mathbb{R}_+$, and total number of clients $M$.  The results are depicted in Figures~\ref{Fig:reg_kl}-\ref{Fig:reg_client_test}. All results are obtained by average of \cblue{$400$} independent Monte Carlo simulations.

% \begin{figure}
% \centerline{\includegraphics[width=0.45\textwidth]{figs/fig_feb/KL_reg.png}
% }
% \caption{Above is the KL wrt the prior variance $c_0$ when the prior mean is fixed, number of clients fixed to 6; 
% below is the KL wrt the number of client, with prior mean $w_0$ when the prior variance is fixed to 9}
% \label{fig:KL_reg}
% \end{figure}

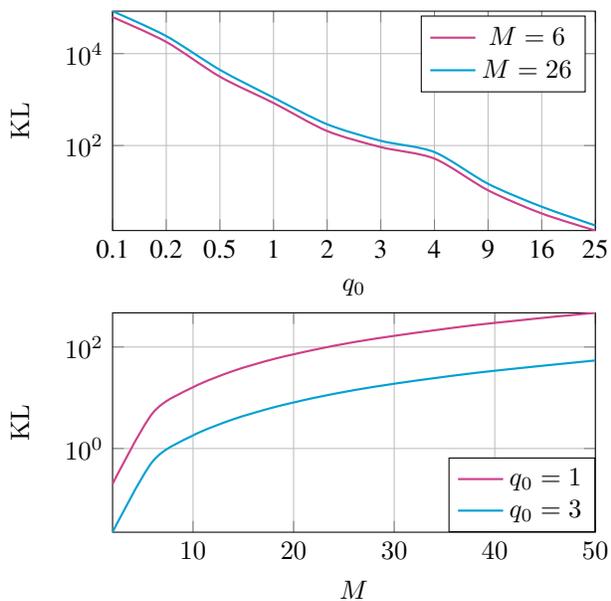
\begin{figure}[t]
    \centering
    \begin{subfigure}[t]{0.4\textwidth}
        \centering
        \input{latex_figures/kl_copy}
        % \caption{Subcaption of the left diagram.}
        \label{Fig:kl}
    \end{subfigure}%
    \vskip 0.1cm
    \begin{subfigure}[t]{0.4\textwidth}
        \centering
        \input{latex_figures/kl_client}
        % \caption{Subcaption of the right diagram.}
        \label{Fig:kl_client}
    \end{subfigure}
    % \captionsetup{justification=centering,margin=2cm}
    \caption{Evolution of the KL divergence 
    % in \eqref{eq:kl_gaussian}
    wrt $q_0$ (top) and number of client $M$ (bottom).}
    \label{Fig:reg_kl}
\end{figure}

Figure \ref{Fig:reg_kl} shows the evolution of the KL divergence with respect to the prior variance $q_0$ (top panel for $M=6$ and $M=26$) and the number of clients $M$ (bottom panel for $q_0=1$, $q_0=3$).
In this figure, the KL decreases wrt prior variance $q_0$. Which can also be shown in the bottom panel for $q_0=1$, $q_0=3$, where the KL increase wrt the number of clients. 
%
%
% The bottom of Figure \ref{Fig:reg_kl} shows the KL divergence of weight distribution is increasing wrt M increasing.
%
These results are consistent with Theorem \ref{theorem_M} and \ref{theorem_q}, which was already predicting such behavior. 

% Figure \ref{fig:KL_sigma_weight8} show the KL divergence of independent fusion and correlated fusion when the prior weight mean is a little far away from the real weight. This figure show the KL increasing first and decrease rapidly. This is because we are very certain about the prior information, it's dominant even in local models.Therefore, the correlated fusion result is similar to independent fusion at the very low prior variance value. As the prior variance increase, the two distribution we become more different. However, as prior variance is less uncertain, the likelihood information will dominant, that's why correlated fusion and independent fusion are almost the same as a very big prior variance. 

% \begin{figure}
%     \centerline{\includegraphics[width=0.5\textwidth]{figs/fig_feb/MSE_reg.png}
%     }
%     \caption{The MSE of test and weight wrt the prior variance $c_0$ when the prior mean is fixed, number of clients is 6}
%     \label{fig:reg_MSE_test_q}
% \end{figure}

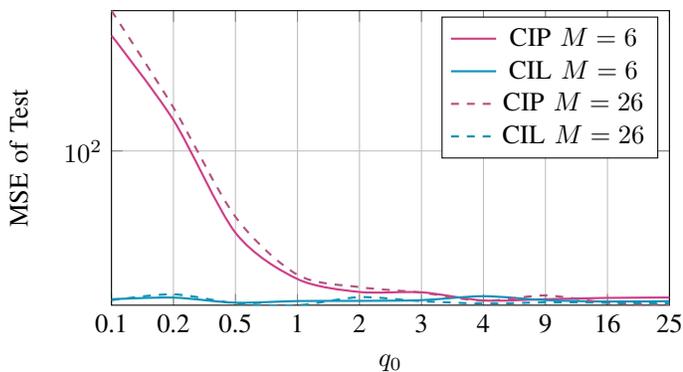
\begin{figure}[t]
    \centering
    \begin{subfigure}[t]{0.5\textwidth}
        \centering
        \input{latex_figures/req_q_test_copy}
        % \caption{Subcaption of the left diagram.}
        % \label{Fig:reg_q_test}
    \end{subfigure}
\caption{Test MSE with respect to $q_0$ for the distributed linear estimation in Section \ref{sec:exa}.}
    \label{Fig:reg_q_test}
\end{figure}
 
% \begin{figure}[h]
% \begin{subfigure}{\textwidth}
% \includegraphics[width=0.45\linewidth, height=6cm]{figs/fig_jan/reg_MSE_test_q.png}
% \caption{Caption1}
% \label{fig:reg_MSE_test_q}
% \end{subfigure}
% \begin{subfigure}{\textwidth}
% \includegraphics[width=0.9\linewidth, height=6cm]{figs/fig_jan/reg_MSE_weight_q.png}
% \caption{Caption 2}
% \label{fig:reg_MSE_weight_q}
% \end{subfigure}

% \caption{Caption for this figure with two images}
% \label{fig:image2}
% \end{figure}

% \begin{figure}[ht]
%   \centering
%   \begin{subfigure}{\linewidth}
%     \centering
%     \includegraphics[width=.95\linewidth]{figs/fig_jan/reg_MSE_test_q.png}
%     % \caption{Caption for image 1}
%   \end{subfigure}

%   \begin{subfigure}{\linewidth}
%     \centering
%     \includegraphics[width=.95\linewidth]{figs/fig_jan/reg_MSE_weight_q.png}
%     % \caption{Caption for image 2}
%   \end{subfigure}  
%   \caption{A caption for both images}  
% \end{figure}  

Figure \ref{Fig:reg_q_test} shows the test MSE wrt the prior variance $q_0$ for a fixed number of collaborating clients of 6 and 26.
This figure shows that all tested configurations converge to similar results when the prior variance is large enough, that is when the a priori information on $\btheta$ is non-informative. However, when given $p(\btheta)$ is informative, represented by smaller variance values, the impact of properly accounting for it would be more apparent. Independently of the number of users, CIL seems to exhibit stable results, 
% while CIP has a larger impact on the result as the prior becomes narrower.
\cblue{whereas CIP's performance degrades as the prior becomes narrower.}
% While BPoEs shown a relatively good results between CIP and CIL. 
% Because the BPoEs don't share the same prior, thus the prior information will not have that effect as prior in CIP.

% \begin{figure}
% \centerline{\includegraphics[width=0.45\textwidth]{figs/fig_jan/reg_MSE_weight_q0.png}
% }
% \caption{the MSE of weight wrt the prior variance $c_0$ when the prior mean is fixed, number of clients is 6.}
% \label{fig:reg_MSE_weight_q0}
% \end{figure}

\begin{figure}[t]
    \centering
    \begin{subfigure}[t]{0.5\textwidth}
        \centering
        \input{latex_figures/reg_client_test}
        % \caption{Subcaption of the left diagram.}
        \label{Fig:reg_client_test}
    \end{subfigure}
    \caption{The test MSE wrt $M$ under different prior information when prior variance at $1$ and $3$.}
    \label{Fig:reg_client_test}
\end{figure}
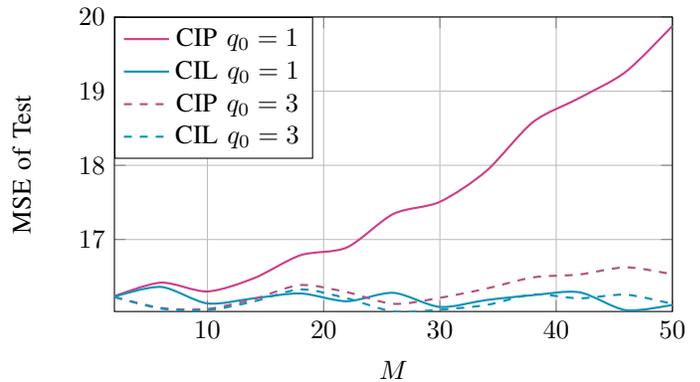

Figure \ref{Fig:reg_client_test} shows the MSE of test wrt the number of clients $M$, as well as for the two representative values prior variance $q_0=1$ and $q_0=3$. 
The results suggest that CIL can outperform CIP as we increasing the number of clients. In that situation, for growing $M$, CIP uses the prior information repeatedly, causing the results to be biased unless this fact is taken into consideration as in the CIL appoach. This effect is even more notorious when the a priori information is not accurate (e.g., the mean value is far from the true value of $\btheta$). %which make sense the results will be biased, especially the prior information is not accurate.

\subsection{Fusion of local class posterior probabilities}\label{sec:exb}

\cblue{We are now interested in analyzing the performance of the Bayesian data fusion approaches when priors are shared across agents in a classification problem.}
Similarly to BCM \cite{tresp2000bayesian}, we consider the fusion of classification results as the fusion of probability functions. Analogous to \eqref{eq:CIL}, the CIL fusion,  given input $\mathbf{x}$, results in a discrete posterior probability distribution:
\begin{align}\label{eq:clf_probability}
P(C=c|\mathcal{D}_1,\dots,\mathcal{D}_M, \mathbf{x}) \propto \frac{\Pi_{m=1}^M P(C=c|\mathcal{D}_m,\mathbf{x})}{P^{M-1}(C=c)},
\end{align}
where $c$ is the classification label with $c\in\{1,\dots,L\}$. The a priori probability of the classes, shared among agents, is denoted by $P(C=c)$. 
To obtain the local class posterior probabilities $P(C|\mathcal{D}_m,\mathbf{x})$, we consider an linear discriminant analysis (LDA) model, which optimally utilizes the shared prior, then fuse the local posteriors to obtain the final result in \eqref{eq:clf_probability}. %The following results show the difference between the CIP and CIL under this scenario.

In this experiment, for the sake of clarity, we are considering an $L=2$ classes classification problem. The objective is thus to estimate their class posterior probabilities $P(C=i|\mathcal{D}_1,\dots,\mathcal{D}_M, \mathbf{x})$ with $i\in\{1,2\}$. 
% We generated a synthetic dataset by setting the two class distribution as $[0.6,0.4]$,
% % distribution of the two classes as $[0.6, 0.4]$. 
% and by sampling from two normal distributions. The total number of data points is $1000$, which were randomly split into multiple clients using a Dirichlet distribution. 
%
We generated a synthetic dataset, consisting of $1000$ data points, by sampling from two normal distributions (for class $i$, the conditional likelihood of the generated data is $p(\mathcal{D}_m|\mathbf{x},C=i) = \mathcal{N}(\mathbf{1}_i,\mathbf{I})$, where $\mathbf{1}_i$ is an all-zeroes vector except for a $1$ at the $i$-th element. The dimension of the data $\mathbf{y}$ is $10$. The prior class probabilities were $0.6$ and $0.4$ for classes $1$ and $2$, respectively. The resulting dataset was divided into $M$ subsets of the same dimension $1000/M$, such that these are \textit{i.i.d.} distributed.

\begin{figure}[t]
    \centering
    \begin{subfigure}[t]{0.5\textwidth}
        \centering
        \input{latex_figures/clf_pro_kl_q}
        % \caption{Subcaption of the left diagram.}
        % \label{Fig:clf_pro_kl_q}
    \end{subfigure}%
    \vskip 0.1cm
    % \vspace{1mm}
    \begin{subfigure}[t]{0.5\textwidth}
        \centering
        \input{latex_figures/clf_pro_kl_client}
        % \caption{Subcaption of the right diagram.}
        % \label{Fig:clf_pro_kl}
    \end{subfigure}
    \caption{KL divergence between CIL and CIP class posteriors as a function of (top) the class $1$ prior probability $P_1$ and (bottom) number of clients $M$. Distributed classification problem from Section \ref{sec:exb}.}
    \label{Fig:clf_pro_kl}
\end{figure}
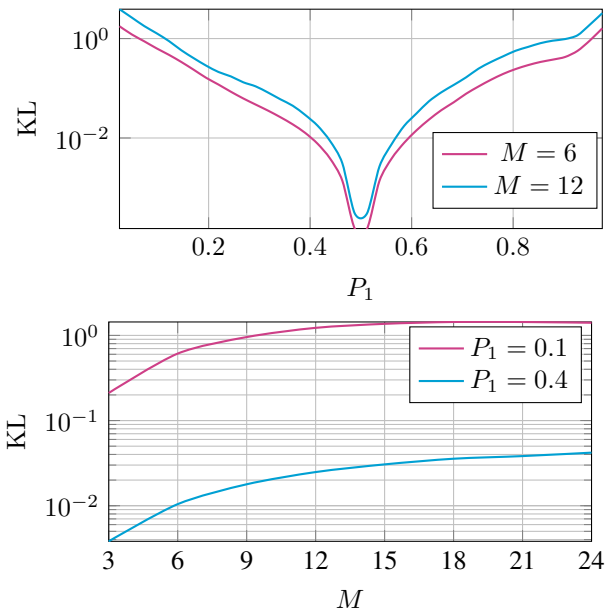

\cblue{Figure \ref{Fig:clf_pro_kl} shows the evolution of the KL divergence between class posteriors of CIP and CIL as function of different systems parameters. The top panel depicts the KL divergence as a function of the assumed (misspecified) prior $P_1\triangleq P(C=1)$ for $M\in\{6,12\}$. The panel shows that the KL divergence is minimized when the assumed prior probabilities are $P_1=P_2=1/2$, such that non-informative priors were employed. % The KL divergence also decreases when the class prior probability is approaching the non-informative prior probability as $[0.5,0.5]$. 
Additionally, the KL divergence for $M=12$ clients is larger than the one for $M=6$ clients, which is also shown in the bottom panel where the KL divergence increases with $M$.  
The bottom panel depicts the KL as a function of $M$ while fixing the assumed priors as $P_1=0.1$ and $P_1=0.4$. Again the plot shows that the KL is smaller for non-informative priors while the KL consistently increases with the number of clients. 
}
%
% Figure \ref{Fig:clf_pro_kl} shows the evolution of the KL divergence of class posterior distribution of CIP and CIL with respect to (top panel) the first class prior probability $P_1\triangleq P(C=1)$ for $M=6$ and $M=12$ clients; and (bottom panel) the number of clients $M$ for $P_1=0.1$ and $P_1=0.4$.
%
% In the top panel, the figure shows that the KL divergence is minimized when the class prior probabilities are $P_1=P_2=1/2$, such that non-informative priors were employed. % The KL divergence also decreases when the class prior probability is approaching the non-informative prior probability as $[0.5,0.5]$. 
% Additionally, the KL divergence for $M=12$ clients is larger than the one for $M=6$ clients, which is also shown in the bottom panel where the KL divergence increases with $M$. 

Figure \ref{Fig:clf_pro_q_acc} shows the accuracy performance of posterior class distribution of both CIP and CIL with respect to \cblue{the assumed prior} $P_1$. \cblue{The accuracy is a metric that measures the percentage of correct predictions over the total test samples observed. % as $\text{accuracy}= \frac{\text{correct predictions}}{\text{Total test samples}}$.
}
Similarly to the KL analysis in Figure \ref{Fig:clf_pro_kl}, the accuracy of both approaches are similar for non-informative class priors $P_1=P_2=1/2$. %We can tell from figure \ref{Fig:clf_pro_q_acc}, their accuracy of CIP and CIL is the almost the same when prior probability $P_1$ is close to 0.5, as the conclusion in the above figure KL divergence analysis. 
On the other hand, when the prior probability is far from the actual class prior probability, CIP's accuracy is degraded since that mismatched prior is overused, whereas for CIL the performance does not drop dramatically. When the prior probability is close to real class distribution (that is, $P_1=0.6$ and $P_2=0.4$), both of these two methods achieve relatively decent results.
%
%
% The CIL give the most stable result even the prior distribution for class is not accurate. However, after the prior approaching the real prior ($P_1=0.6$), the Global and CIP can achieve the good performance. The CIL is trying to cancel out the prior information for double counting, while the CIP reuse the prior information many times, which make sense when prior information is good enough, the CIP will also get a good performance. 
%
% Figure \ref{Fig:clf_pro_kl} below show the KL divergence of CIL and CIP wrt the number of Clients. As the number increasing, they are less similar to each other. 
%
%
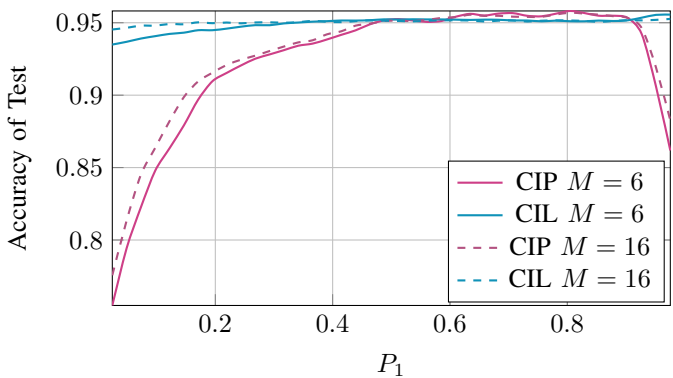
\begin{figure}[t]
    \centering
    \begin{subfigure}[t]{0.5\textwidth}
        \centering
        \input{latex_figures/clf_pro_q_acc}
        % \caption{Subcaption of the left diagram.}
        \label{Fig:clf_pro_q_acc}
    \end{subfigure}
    \caption{The accuracy with respect to the class 1 prior probability for various numbers of clients $M$ and fixed amount of total data.}
    \label{Fig:clf_pro_q_acc}
\end{figure}
Figure \ref{Fig:clf_pro_client_acc} shows the classification accuracy, this time in terms of the number of clients $M$ and two values for $P_1$.
The results show that CIL outperforms CIP, similarly to Figure \ref{Fig:clf_pro_q_acc}. %In this figure, the CIL still can achieve better and more stable results than CIP. Even though we don't saw the trend of result as in regression that performance is decreasing wrt M, which can be account to the model complexity. As LDA is a linear classifier, increasing the number of clients, mean we are using more linear classifiers together decide the boundary of the class, which make the whole model more complex and accurate. 

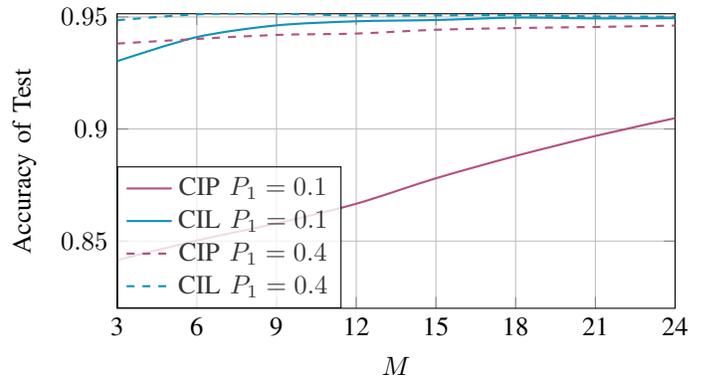
\begin{figure}[t]
    \centering
    \begin{subfigure}[t]{0.5\textwidth}
    % textwidth}
        \centering
        \input{latex_figures/clf_pro_client_acc}
        % \caption{Subcaption of the right diagram.}
        \label{Fig:clf_pro_client_acc}
    \end{subfigure}
    \caption{Testing accuracy with respect to $M$ under class 1 prior probabilities $0.1$ and $0.4$, for the classification problem in Section \ref{sec:exb}.}
    \label{Fig:clf_pro_client_acc}
\end{figure}

\subsection{Distributed learning of neural network classifiers}\label{sec:exc}

% \begin{itemize}
% \item Neural network structure
% \item  Laplace Approximation
% \item discuss some limitation to use correlated fusion in NN.
% \end{itemize}

In this test we consider a set of $M$ neural networks with the same structure and network parameters $\btheta$, each trained on observed local data. The objective is therefore to fuse the trained NNs into a global NN that can then be used for classification purposes, accounting for all local data through a distributed training process. 
\cblue{More precisely, the fusion is among the locally learned $\btheta$, as opposed to fusing the local class posterior as done in the experiments in Section \ref{sec:exb}.}
Additionally, since the focus of this paper is on Bayesian data fusion, those are Bayesian NNs (BNN) in that their weights are treated as random variables. Therefore, training of a BNN involves inferring the joint posterior of its parameters given available data. The objective of this experiment is to locally train the BNNs, then fuse the local posteriors \cblue{$p(\btheta|\data_m)$} to compute a global posterior $p(\btheta|\mathcal{D}_1,\dots, \mathcal{D}_M)$. The use of Bayesian approaches within the context of data-driven models has been investigated previously and is given attention more recently \cite{fortuin2021bayesian}, where for the sake of simplicity the weights in the BNN are typically assumed independent and normally distributed such that mean and covariance can characterize their posterior.
%How to introduce the Bayes into neural network recently arising a much attention\cite{fortuin2021bayesian}. For simplicity, usually assuming the weights in Neural network as independent Gaussian distributed. 

% In this paper, we use the whole covariance of matrix of weight because the neural network structure in our paper is not large, which also can give us a more reliable result. 
In this experiment, we used a similar classification dataset as in Section \ref{sec:exb}, sampling from a normal distribution for each class $p(\mathcal{D}_m|\mathbf{x},C=i) = \mathcal{N}(\mathbf{1}_i,\mathbf{I})$, where in this case the dimension of the problem was increased to $10$ classes, $i=\{1,\dots,10\}$, while size of the data was kept to $10$ as in Section \ref{sec:exb}. \cblue{The prior was normally distributed as $p(\btheta)=\mathcal{N}(\mathbf{0},q_0\mathbf{I})$, with $q_0$ adjusting its variance}. The BNN trained by the $M$ local agents had a hidden layer of $64$ neurons, the training epoch was set to $100$, and the learning rate was $0.05$. 
It is worth mentioning that the total amount of data is fixed, such that increasing $M$ would have the impact of reducing the amount of local data available at each client.
Given a local dataset, there are several methods that can be used to calculate the posterior distribution of the BNN parameters. Without loss of generality, we considered here a Laplace approximation of the BNN parameters \cite{bishop2006pattern}, assuming the parameters are normally distributed. 

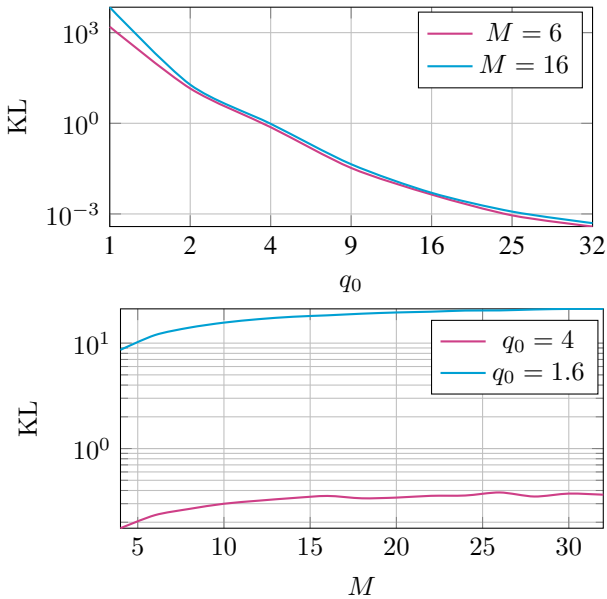
\begin{figure}[t]
    \centering
    \begin{subfigure}[t]{0.5\textwidth}
        \centering
        \input{latex_figures/clf_kl_q_2}
        % \caption{Subcaption of the left diagram.}
        % \label{Fig:reg_q_test}
    \end{subfigure}%
    \vskip 0.1cm
    \begin{subfigure}[t]{0.5\textwidth}
        \centering
        \input{latex_figures/clf_kl_client_2}
        % \caption{Subcaption of the right diagram.}
        \label{Fig:kl_client}
    \end{subfigure}
    \caption{KL divergence between CIL and CIP class posteriors as a function of (top) $q_0$ and (bottom) $M$ for the distributed classification training in Section \ref{sec:exc}.}
    \label{Fig:clf_kl}
\end{figure}

Figure \ref{Fig:clf_kl} shows the KL divergence between the parameters' posterior of both CIP and CIL solutions, as a function of the a priori variance $q_0$ (top panel for $M=6$ and $M=16$ clients) and the number of clients $M$ (bottom panel for $q_0=4$ and $q_0=1.6$).
From the top panel of this figure, we can see that behavior of the KL divergence with respect to $q_0$ or $M$ are similar to those from \ref{sec:exa}. % still hold that KL divergence decreasing as the variance increasing and increasing as the number of client increasing. 

\begin{figure}[t]
    \centering
    \begin{subfigure}[t]{0.5\textwidth}
        \centering
        \input{latex_figures/clf_q_acc_2}
        % \caption{Subcaption of the left diagram.}
        \label{Fig:clf_q_acc}
    \end{subfigure}
    \caption{Testing accuracy with respect to the prior variance for a fixed amount of total data across the $M$ clients.}
    \label{Fig:clf_q_acc}
\end{figure}
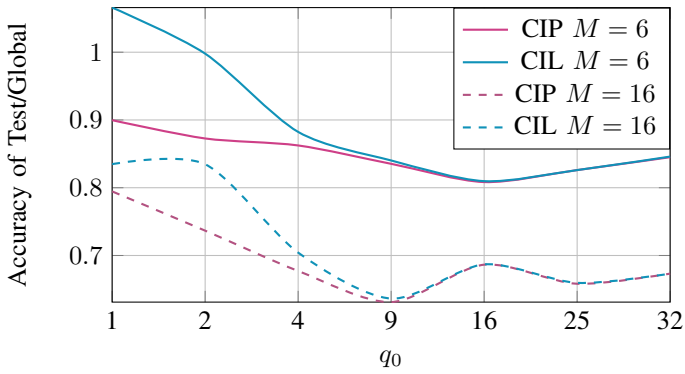

\begin{figure}[t]
    \centering
    \begin{subfigure}[t]{0.5\textwidth}
        \centering
        \input{latex_figures/clf_client_acc_2}
        % \caption{Subcaption of the right diagram.}
        \label{Fig:clf_client_acc}
    \end{subfigure}
    \caption{Testing accuracy with respect to $M$ for a different values of $q_0$.}
    \label{Fig:clf_client_acc}
\end{figure}
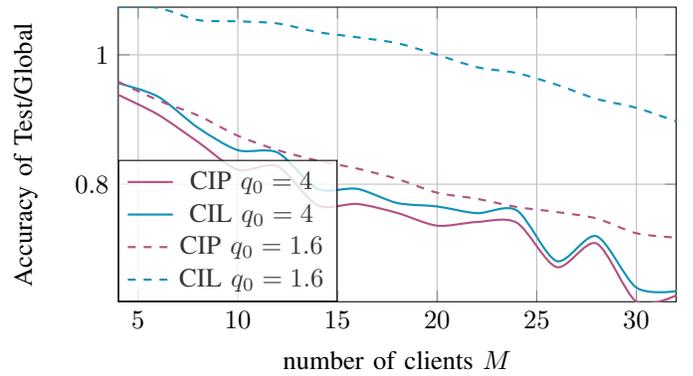

Figure \ref{Fig:clf_q_acc} shows the fusion center NN testing accuracy with respect the prior variance $q_0$ for $M=6$ and $M=16$ clients. For the sake of clarity, and to facilitate the identification of the trend behavior, the reported metric is test accuracy of CIL (and CIP) divided by global accuracy (considered as a benchmark, using the global model for testing, which is trained by using all datasets jointly). As expected, as the prior becomes more informative, CIL outperforms CIP, and as $M$ increases the local data becomes more scarce and the performance degrades.
Analogously, Figure \ref{Fig:clf_client_acc} shows the testing accuracy with respect to the number of clients $M$ for different values of $q_0$, drawing similar conclusions. %The metric of y-axis set to test accuracy of CIP,CIL that divided by Global accuracy.

\subsection{Recursive update of classification problem using $M$ neural networks}\label{sec:exd}

Results presented in the previous subsection only considered the case where classifier's parameters are fused once, after learned using local data. In many applications, however, one aims at fusing models over time in an iterative procedure. For instance, this is the case of many distributed and/or federated methods where a common prior can be shared with several nodes and fused back recursively~\cite{liu2021bayesian}.
% \cblue{[add refs here]}.
% Above are only one-time fusion, in real practice, recursively update the prior is more common. Typically used in the Federated learning, the parameters information will be aggregated in the fusion center and pass back to the local. 
In this experiment, we consider the case of Bayesian federated model learning where learned distributions from distributed clients are fused by a central node in the vein of Figure~\ref{fig:distributed}. 
% about the aggregate the parameters distribution instead of just parameter, and pass back the fusion result to the local as prior information. 
Locally, at the $t$-th communication round each client obtains their local posterior of the BNN model parameters using Bayes theorem as
% Let $p_{t-1}(\boldsymbol{\theta})$ be the posterior provided by the fusion node at time $t-1$ and $p(\mathcal{D}_m|theta)$
% The equation \ref{eq:recursive} show how local posterior use the likelihood and prior information from the fusion center. 
\begin{align}\label{eq:recursive}
p_t(\boldsymbol{\theta}|\mathcal{D}_m)=& \frac{p_{t-1}(\boldsymbol{\theta} )  p(\mathcal{D}_m | \boldsymbol{\theta})}{P(\mathcal{D}_m)},
\end{align}
\noindent where $p_{t-1}(\boldsymbol{\theta} )$ denotes the prior distribution of the parameters at $t$ and the $\log$-posterior can be written as
\begin{align}
\ln p_t(\boldsymbol{\theta}|\mathcal{D}_m) &=\ln p_{t-1}(\boldsymbol{\theta}) +\ln p(\mathcal{D}_m | \boldsymbol{\theta})-\ln p(\mathcal{D}_m) \, . \label{eq:recursive_last}
\end{align}

For simplicity, we consider the Gaussian assumption on the posterior $p_t(\boldsymbol{\theta}|\mathcal{D}_m) \approx \mathcal{N}(\btheta_{m,t},\mathbf{C}_{m,t})$, similarly as in \eqref{eq:local_posterior}.
Defining the loss for the $m$-th client as the log-posterior in~\eqref{eq:recursive_last}, i.e., $J_{m,t}(\btheta) = \ln p_t(\boldsymbol{\theta}|\mathcal{D}_m)$, and using the Laplace approximation for its second term, the loss can be approximated as
\begin{align}
 J_{m,t} (\boldsymbol{\theta}) \approx &\frac{1}{2}(\boldsymbol{\theta}-\mathbf{\boldsymbol{\mu}}_{t-1})^{\top} \mathbf{\Lambda}_{t-1}(\btheta-\mathbf{\boldsymbol{\mu}}_{t-1}) 
\\ \nonumber
% & -\ln P(D_m |\btheta_{m,t})
+&\frac{1}{2}(\btheta-\btheta_{m,t}^{\textrm{ML}})^{\top} \mathbf{H}_{m,t}(\btheta-\btheta_{m,t}^{\textrm{ML}})+\kappa, 
\end{align}
\noindent where $p_{t-1}(\boldsymbol{\theta}) \approx \mathcal{N}(\mathbf{\boldsymbol{\mu}}_{t-1},\mathbf{\Lambda}_{t-1}^{-1})$ is the parameter prior at $t$, resulting from the CIL (or CIP) fusion of local posteriors at $k-1$. 
The Laplace approximation of the likelihood
$p(\mathcal{D}_m | \boldsymbol{\theta}) \approx \mathcal{N}(\btheta_{m,t}^{\textrm{ML}} , \mathbf{H}_{m,t}^{-1})$ requires maximum likelihood (ML) estimation of the parameter, $\btheta_{m,t}^{\textrm{ML}}$, possibly through a gradient method or other numerical optimization approach. A suitable estimate of the inverse covariance $\mathbf{H}_{m,t}^{-1}$ in the Laplacian approximation is known to be \cite{bishop2006pattern} the Fisher information matrix $\mathcal{I}(\btheta)$, which can be approximated \cite{liu2021bayesian} by 
$\mathcal{I}_m (\btheta) \approx \frac{1}{\left|\mathcal{D}_m\right|} \sum_{(\mathbf{x}, \mathbf{y}) \in \mathcal{D}_m} \nabla_\theta \log p(\mathbf{y} | \mathbf{x}, \btheta) \nabla_\theta \log p(\mathbf{y} | \mathbf{x}, \btheta)^{\top}$ for the $m$-th client.
%where $\btheta_{m,t}$ is the mean of the local posterior for the $m$-th client at time $t$  obtained as described in~\eqref{eq:local_posterior}. 
The $\kappa$ gathers constant terms that are not related to $\btheta$ and thus do not contribute to the loss minimization.
Then, the local parameter posterior mean for the $m$-th client can be estimated by solving the following optimization problem:
\begin{equation}\label{eq:mean of weight recursively}
    \btheta_{m,t} = \mathop{\arg\max}_{\btheta} J_{m,t}(\btheta) \;,
\end{equation}
\noindent and the associated local covariance is \cbluee{\cite{liu2021bayesian}}

\begin{align}\label{eq:variance of weight recursively}
\mathbf{C}_{m,t}^{-1}=\frac{\partial^2 J_m(\btheta)}{\partial \btheta \partial \btheta^\top} & \approx \cblue{\mathcal{I}}_{m}(\btheta_{m,t}^{\textrm{ML}})+\mathbf{\Lambda}_{t-1}\,.
\end{align}

After obtaining the local posterior parameters at $t$, $\btheta_{m,t}$ and $\mathbf{C}_{m,t}$, for $m=1,\ldots,M$, these can be fused using either CIL and CIP approaches. This fused posterior would become $p_{t}(\boldsymbol{\theta}) \approx \mathcal{N}(\mathbf{\boldsymbol{\mu}}_{t},\mathbf{\Lambda}_{t}^{-1})$ the new prior used at the next communication round $t+1$. %The accuracy of both resulting classifiers, using CIL and CIP, and the evolution of $\mathca{D}_\mathrm{KL}$ of their distributions are summarized in Figures~\ref{Fig:clf_recursive_acc} to Figure~\ref{Fig:clf_recursive_kl}.

Compared to previous classification experiments, this experiment considers a more complex synthetic dataset generated from a Gaussian mixture with a maximum of $4$ Gaussian components, with $3$ possible classes and $10$ features dimension. A total of $600$ \textit{i.i.d.} training points and $300$ testing points were generated. The considered BNN model is a multi-layer perceptron (MLP) with two hidden layers, containing $32$ and $8$ neurons respectively, with a learning rate of $0.01$ at training.

\begin{figure}[t]
    \centering
    \begin{subfigure}[t]{0.5\textwidth}
        \centering
        \input{latex_figures/clf_recursive_acc}
        % \caption{Subcaption of the right diagram.}
        \label{Fig:clf_recursive_acc}
    \end{subfigure}
    \caption{Testing accuracy evolution with communication round for the recursive classification learning in Section \ref{sec:exd}.}
    \label{Fig:clf_recursive_acc}
\end{figure}
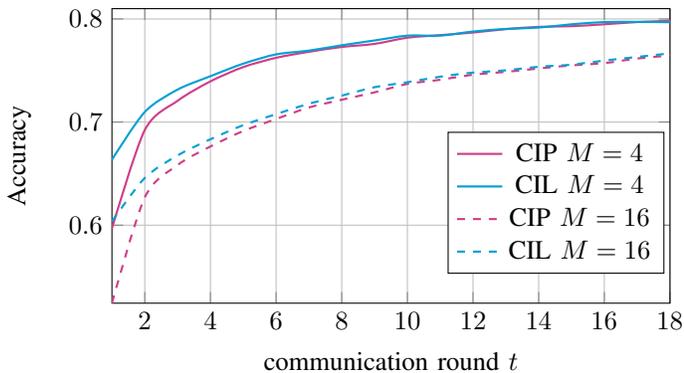

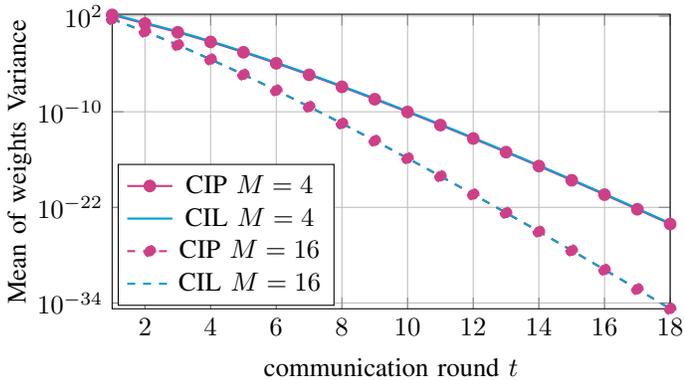
\begin{figure}[t]
    \centering
    \begin{subfigure}[t]{0.5\textwidth}
        \centering
        \input{latex_figures/clf_recursive_var_mean}
        % \caption{Subcaption of the right diagram.}
        \label{Fig:clf_recursive_var_mean}
    \end{subfigure}
    \caption{Evolution over $t$ of the mean variance of model parameters.}
    \label{Fig:clf_recursive_var_mean}
\end{figure}

% The same with KL divergence, it decrease with respect to the communication rounds.

% \begin{figure}[t]
%     \centering
%     \begin{subfigure}[t]{0.5\textwidth}
%         \centering
%         \input{latex_figures/clf_recursive_kl}
%         % \caption{Subcaption of the right diagram.}
%         \label{Fig:clf_recursive_kl}
%     \end{subfigure}
%     \caption{Evolution of KL divergence wrt communication round for $M\in\{4,16\}$.}
%     \label{Fig:clf_recursive_kl}
% \end{figure}
% \cblue{final decision: remove this figure \ref{Fig:clf_recursive_kl}?}

The results in Figures \ref{Fig:clf_recursive_acc} and \ref{Fig:clf_recursive_var_mean} show the evolution of accuracy and the mean value of the variance of all parameters, as a function of the communication round $t$ at which fusion happens. $M=4$ and $M=16$ are tested.
In the experiments, the initial round is given with non-informative prior information. After several communicate rounds occur, the variance becomes very informative, as shown in Figure \ref{Fig:clf_recursive_var_mean}. %, the weight prior variance will be more deterministic wrt more communication rounds.

It is seen from Figure \ref{Fig:clf_recursive_acc} that accuracy generally increases with the number of communication rounds, for both CIP and CIL. Also, the accuracy decreases inversely proportional to the number of clients $M$. The latter is a consequence of having less local data samples at each node, given that the overall amount is kept constant regardless of $M$. Another conclusion from the experiment is that CIL outperforms CIP, approaching each other when enough communication rounds are run. On the other hand, the average variance of the parameters produced by \eqref{eq:variance of weight recursively} is negligible, both approaches becoming over-confident. %We were trying to show the MSE results of weights' mean, but always gets 0 output just after several communication rounds, the weights in Neural network is too small, which make it reasonable that we can't tell the difference by MSE, instead we using the KL divergence shown in \ref{Fig:clf_recursive_kl}. By equation \ref{eq:kl_gaussian}, KL divergence results is more account to the weights' variance than the mean of weight as the third term is close to 0. It's interesting to see though that KL divergence is stable after several communication rounds.  

\section{Conclusions}\label{sec:conclusions}
This paper investigates the impact of sharing a priori information of parameters of interest in a Bayesian data fusion context. The focus is on distributed learning of models for regression and classification purposes. We \cblue{analyze} two fusion possibilities in a distributed system, so-called conditionally independent posterior (CIP) and conditionally independent likelihoods (CIL). CIL assumes the local datasets are mutually conditionally independent and CIP is based on the approximation that the local posteriors are mutually independent. CIP is sometimes desirable due to its simplicity, however the assumption is compromised when prior information is shared among clients for which CIL is better suited. 
This article provided an analysis of both fusion approaches and their performance as a function of the uncertainty of the shared prior and the number of agents in the distributed setting. A deeper analysis was provided under the assumption that local posteriors can be approximated as Gaussian distributed. It is shown that both CIL and CIP converge to similar fusion results when the prior becomes non-informative. Another result from the analysis shows that, for the same amount of global data, when the number of collaborative agents increase the CIP solution degrades with respect to CIL. Or in other words, 
%The paper quantifies the mismatch for the Gaussian case comparing those approaches and showing that they become equivalent when the a priori information is either not informative or 
when the number of users decrease, both approaches become equivalent. %in which case more data is used locally and the prior becomes less relevant. 
A comparison of these two methods in different applications of distributed inference was provided. %See the trend of KL divergence and results wrt the setting of initial variance and number of clients while keep the total number of Dataset fixed. The KL divergence and results will converge to each other while we using the large variance which mean prior information have less effect on the results. When increasing the number of clients, KL divergence is usually increase and the results become worse. However, some situations may not apply as shown in experiment \ref{sec:exb} because of the model complexity. 
Namely, experimental results were discussed for distributed learning of linear regression and classification models; distributed learning of neural network classifiers; and a practical use case in the context of federated learning, where the prior information is updated recursively over time as opposed to only once. The results consistently show that CIL, which accounts for the shared prior information, generally outperforms CIP, as it was predicted by the analytical results presented in this paper. 

% In the future, we are going to explore the problem in larger machine learning problem. Plus, we will introduce AA and GA fusion strategy to deal with the data double counting problem. More, the personalized Federated Learning to deal with the Non-IID problem is also worth to explore.

% or


\appendices
% \appendix

\section{Bayesian data fusion under Gaussian distributions}\label{ap:GaussianCIL}

% \pau{add here the derivation for mu and Lambda}

% The optimal Bayesian update, or 
CIL, in \eqref{eq:CIL} is composed of three terms. Namely, 
$i)$ the Product of Experts term, $\prod_{m=1}^{M} p(\btheta|\mathcal{D}_m)$, which we refer to as the CIP; %fusion rule in \eqref{eq:gassuian_Correlated fusion_approx} for the Gaussian case;
$ii)$ the product of priors, $p^{M-1}(\btheta)$; and
$iii)$ the normalizing constant $Z$, which has no impact in obtaining the result we are interested in this appendix.

Under the Gaussian
assumption, 
the first term is given by %the equation \ref{eq:gassuian_Correlated fusion} is seen as three part, the first part $\prod_{m=1}^{M} p(\btheta|\mathcal{D}_m)$ is same as CIP in equation 
\eqref{eq:gassuian_Correlated fusion_approx} following known results of products of Gaussian distributions \cblue{\cite{bromiley2003products}}, such that%, and get the result in \ref{eq:mean_CIP} as:
\begin{align}\label{eq:appendix_agents_product}
\prod_{m=1}^{M} p(\btheta|\mathcal{D}_m) \propto  \mathcal{N}(\widetilde{\boldsymbol{\mu}},\widetilde{\mathbf{\Lambda}}^{-1})\;.
\end{align}
% \begin{align}\label{eq:appendix_cip}
%     \prod_{m=1}^{M} p(\btheta|\mathcal{D}_m) = \mathcal{N}(\widetilde{\mathbf{\boldsymbol{\mu}}}, \widetilde{\mathbf{\Lambda}}^{-1})\;,
% \end{align}

The second term is
\begin{align}\label{eq:appendix_prior_product}
    p^{M-1}(\btheta) \propto  \mathcal{N}(\btheta_0,\frac{\mathbf{C}_0}{M-1})\;,
\end{align}
\noindent given that $\btheta\sim p(\btheta) = \mathcal{N}(\mathbf{\btheta}_0,\mathbf{C}_0)$.

The target distribution, $p(\btheta|\mathcal{D}) = 
\mathcal{N}(\mathbf{\boldsymbol{\mu}}, \mathbf{\Lambda}^{-1})$, is also a Gaussian with parameters $\mathbf{\boldsymbol{\mu}}$ and $\mathbf{\Lambda}$ which can be computed from 
\begin{equation}\label{eq:CIL_division}
\mathcal{N}(\mathbf{\boldsymbol{\mu}}, \mathbf{\Lambda}^{-1}) \propto
\frac{
\mathcal{N}(\btheta_0,\frac{\mathbf{C}_0}{M-1})}{\mathcal{N}(\widetilde{\boldsymbol{\mu}},\widetilde{\mathbf{\Lambda}}^{-1})},
\end{equation}
\noindent since the division of two Gaussian distributions is yet another Gaussian distribution up to a normalizing constant, since it belongs to the exponential family.
%The third part is the normalized constant, which doesn't affect the mean and covariance. As Gaussian distribution is belong to exponential family, one Gaussian divide by another one with the same random variable is still a Gaussian, and then normalized by the constant.  

To calculate the desired mean and covariance, we can use the completing the square method, whereby the exponent in the Gaussian distribution is expanded as
\begin{align}\label{eq:quadratic}
    -\frac{1}{2}(\btheta - \boldsymbol{\mu})^\top\mathbf{\Lambda}(\btheta - \boldsymbol{\mu}) = -\frac{1}{2}\btheta^\top \mathbf{\Lambda} \btheta + \btheta^\top \mathbf{\Lambda} \boldsymbol{\mu} + {\rm const.},
\end{align}
\noindent where the symmetries of $\mathbf{\Lambda}$ are used and the constant represents the terms that are unrelated to $\btheta$.

The exponent in the right-hand side of \eqref{eq:CIL_division} is composed of two quadratic terms of the form in \eqref{eq:quadratic}, that is
%The equation \ref{eq:appendix_agents_product} divide by the equation \ref{eq:appendix_prior_product}
, consider the quadratic part:
\begin{multline}
    -\frac{1}{2}(\btheta - \widetilde{\boldsymbol{\mu}})^\top\widetilde{\mathbf{\Lambda}}(\btheta - \widetilde{\boldsymbol{\mu}})  
    + \frac{1}{2}(\btheta - \btheta_0)^\top(M-1)\mathbf{C}^{-1}_0(\btheta - \btheta_0),
\end{multline}
\noindent which can be rearranged in the form of the right-hand side of \eqref{eq:quadratic} as:
\begin{multline}\label{eq:CIL_app_result}
    -\frac{1}{2}\btheta^\top \mathbf{\Lambda} \btheta + \btheta^\top \mathbf{\Lambda} \boldsymbol{\mu} 
    =-\frac{1}{2}\btheta^\top \left(\sum_{m=1}^{M}\mathbf{C}_m^{-1}-\mathbf{C}_{0}^{-1}(M-1) \right) \btheta \\ 
    +\btheta^\top  \left(\sum_{m=1}^{M}\mathbf{C}_m^{-1} \btheta_m -(M-1) \mathbf{C}_{0}^{-1} \btheta_{0}\right),
\end{multline}
\noindent where we used \eqref{eq:mean_CIP} and \eqref{eq:cov_CIP}. From \eqref{eq:CIL_app_result} we can identify the desired expressions for $\boldsymbol{\mu}$ and $\mathbf{\Lambda}$ in \eqref{eq:mean_CIL} and \eqref{eq:variance_close} respectively.

\section{Proof of Theorem \ref{theorem_q}}\label{ap:thm_q}

% \pau{provide the proof details here for KL asymptotic behavior wrt to prior variance}

% From the equation \ref{eq:local_posterior}, we can see that if prior is non-informative, meaning the $\mathbf{C}_0$ is very large, the likelihood information is the main part. The same with fusion result as shown in the equations  \ref{eq:variance_close}, \ref{eq:mean_CIL}, this part $\mathbf{C}_{0}^{-1}(M-1)$ will approaching 0. Thus, as $q_0\rightarrow \infty$, the result of CIP and CIL will be the same, and it KL divergence will equal to 0. 

\cblue{
This appendix provides a general proof for the result by using arbitrary divergences. First, we discuss a general result that is independent of the assumed posterior distribution. Secondly, a complementary result under the Gaussian assumption is provided to provide more intuition.}

\cblue{
Given the definitions in \eqref{eq:CIL} and \eqref{eq:CIP}, the result is obtained straightforwardly by realizing that when the a priori $p(\btheta)$ distribution is non-informative, this prior becomes independent of $\btheta$. As a consequence, the factor $\frac{p^{M-1}(\boldsymbol{\theta})}{Z}$ in \eqref{eq:CIL} becomes a normalizing constant that, necessarily, has to be the same as $\wt Z$ in \eqref{eq:CIP} for it to be a proper distribution. Therefore making CIL and CIP identical, which would indeed make any divergence between both tend to $0$. 
}

\cblue{The results above are valid for any assumed posterior distribution. However we are providing a complementary results here under the Gaussian posterior assumption, with the objective of providing additional intuititions and connecting to the case in Section \ref{sec:GaussianFusion}.}
Given the a priori distribution $\btheta\sim p(\btheta) = \mathcal{N}(\mathbf{\btheta}_0,\mathbf{C}_0)$ and under the assumption that $\mathbf{C}_0 = q_0\mathbf{I}$, we can readily see from the results in Appendix \ref{ap:GaussianCIL} that the CIL posterior distribution $p(\btheta|\mathcal{D})=\mathcal{N}(\mathbf{\boldsymbol{\mu}}, \mathbf{\Lambda}^{-1})$ results in a mean of 
\begin{equation}\label{eq:mean_CIL_q0}
    \boldsymbol{\mu} = \mathbf{\Lambda}^{-1}\left(\sum_{m=1}^{M}\mathbf{C}_m^{-1} \btheta_m - \frac{M-1}{q_0} \btheta_{0}\right),
    % {\sum_{m=1}^{M}\mathbf{C}_m^{-1}-\mathbf{C}_{0}^{-1}(M-1)} 
\end{equation}
\noindent and a precision matrix of
\begin{align}\label{eq:cov_CIL_q0}
\mathbf{\Lambda} = \sum_{m=1}^{M}\mathbf{C}_m^{-1}- \frac{M-1}{q_0} \mathbf{I} \;.
\end{align}

The result investigates the asymptotics of having a prior information that becomes non-informative, which for the choice of $p(\btheta)$ here corresponds to the case where $q_0\rightarrow \infty$ such that the Gaussian is wider. Therefore,
\begin{eqnarray}
\lim_{q_0\rightarrow \infty} \boldsymbol{\mu} &=& \left( \sum_{m=1}^{M}\mathbf{C}_m^{-1} \right)^{-1} \sum_{m=1}^{M}\mathbf{C}_m^{-1} \btheta_m, \\
\lim_{q_0\rightarrow \infty} \mathbf{\Lambda} &=&  \sum_{m=1}^{M}\mathbf{C}_m^{-1},
\end{eqnarray}
\noindent and we can easily identify that $\underset{q_0\rightarrow \infty}{\lim} \boldsymbol{\mu} = \widetilde{\boldsymbol{\mu}}$ and 
$\underset{q_0\rightarrow \infty}{\lim} \mathbf{\Lambda} = \widetilde{\mathbf{\Lambda}}$
from \eqref{eq:mean_CIP} and \eqref{eq:cov_CIP}, respectively. Since under the Gaussian assumption both CIL and CIP can be parameterized by their means and covariances, we can write
\begin{equation}
\lim_{q_0\rightarrow \infty} p(\btheta|\mathcal{D}) = 
\wt p(\btheta|\mathcal{D})  \;,
\end{equation}
\noindent since those statistics are asymptotically equivalent as shown in \eqref{eq:mean_CIP} and \eqref{eq:cov_CIP}. As a consequence, their divergence tends to zero with $q_0\rightarrow \infty$, independently of $M$, which proves the result.

\section{Proof of Theorem \ref{theorem_M} under Kullback-Leibler divergence}\label{ap:thm_M}

%\section{}
%\appendix[Kullback-Leibler (the other direction)]
\def \btheta {\boldsymbol{\theta}}
\def \thetab {\boldsymbol{\theta}}
\def \data {\mathcal{D}}
\def \wt {\widetilde}
% \cblue{This section is about the detail derivation for Equations \ref{eq:CIL}, \ref{eq:CIP} and its divergences prove with Theorem \ref{theorem_M}.}

\cblue{Recall that the true posterior (referred to as CIL) was expressed in \eqref{eq:CIL} as}
\begin{align}
    p(\btheta|\data) = \frac{\prod_{m=1}^M p(\data_m|\btheta) p(\btheta)}{p(\data)},
\end{align}
where $p(\data) = \int \prod_{m=1}^M p(\data_m|\btheta) p(\btheta)d\btheta$ and that the approximate posterior \cblue{(denoted as CIP) was given in \eqref{eq:CIP} as}
\begin{align}
    \wt p(\btheta|\data) &= \wt Z \prod_{m=1}^{M} p(\boldsymbol{\theta}|\mathcal{D}_m) 
    % \\\nonumber
    = \frac{\prod_{m=1}^M p(\data_m|\btheta) p^M(\btheta)}{p_M(\data)},
\end{align}
where $p_M(\data) = \int \prod_{m=1}^M p(\data_m|\btheta) p^M(\btheta) d\btheta$ and $\wt Z = \frac{\prod_{m=1}^{M}p(\mathcal{D}_m)}{p_M(\data)}$.
With those definitions, let us define the KL divergence between both fusion rules as $\mathrm{KL}_M$ and manipulate it further
% {\small 
\begin{align}\label{eq:kl_m}
\mathrm{KL}_M &\triangleq  \mathrm{KL}_M(p(\btheta|\data)\parallel\wt p(\btheta|\data)) \\ \nonumber
 &= \int p(\btheta|\data) \log\left(\frac{p(\btheta|\data)}{\wt p(\btheta|\data)}\right)d\btheta\\ \nonumber
  &= \int  \frac{\prod_{m=1}^M p(\data_m|\btheta) p(\btheta)}{p(\data)}  \log\left(\frac{ \frac{\prod_{m=1}^M p(\data_m|\btheta) p(\btheta)}{p(\data)}}{ \frac{\prod_{m=1}^M p(\data_m|\btheta) p^M(\btheta)}{p_M(\data)} }\right)d\btheta\\ \nonumber
  &= \int  \frac{\prod_{m=1}^M p(\data_m|\btheta) p(\btheta)}{p(\data)}  \log\left(\frac{p_M(\data)}{p(\data)}\frac{   1 }{   p^{M-1}(\btheta) }\right)d\btheta\\ \nonumber
  &= \int  \frac{\prod_{m=1}^M p(\data_m|\btheta) p(\btheta)}{p(\data)}  \\ \nonumber
  & \cdot
   \left( \log\left(\frac{p_M(\data)}{p(\data)}\right)  - (M-1)\log(p(\btheta))  \right)   d\btheta
  \\ \nonumber
  &=\log\left(\frac{p_M(\data)}{p(\data)}\right)  + (M-1) H(p(\btheta|\data) ,p(\btheta)),
\end{align}
% }
\noindent We also notice that the last term is the cross-entropy of the prior relative to the posterior distribution:
\begin{equation}
    H(p(\btheta|\data) ,p(\btheta)) = - \int  \frac{p(\data|\btheta) p(\btheta)}{p(\data)}    \log(p(\btheta))    d\btheta  ,
\end{equation}
\noindent such that $H(p(\btheta|\data) ,p(\btheta))>0$.

We aim at showing that $\mathrm{KL}_{M+1} > \mathrm{KL}_{M}$, for which we will show that the following quantity is positive:
% {\small 
\begin{align}\label{eq:kl_diff_m}
& \mathrm{KL}_{M+1} - \mathrm{KL}_{M} \\ \nonumber 
  &= \log\left(\frac{p_{M+1}(\data)}{p_{M}(\data)}\right) + H(p(\btheta|\data) ,p(\btheta)) \\ \nonumber
  % &= \log\left(\frac{p_{M+1}(\data)}{p_{M}(\data)}\right) + H(p(\btheta|\data) ,p(\btheta)) \\ \nonumber
  &= \log \int p(\data|\btheta) p^{M+1}(\btheta) d\btheta \\ \nonumber
  &- \log \int p(\data|\btheta) p^M(\btheta) d\btheta + H(p(\btheta|\data) ,p(\btheta)) \\ \nonumber  
  &= \cblue{\log \int \frac{p(\data|\btheta)p(\btheta)}{p(\data)} p^{M}(\btheta) d\btheta} \\ \nonumber
  &\cblue{- \log \int \frac{p(\data|\btheta)p(\btheta)}{p(\data)} p^{M-1}(\btheta) d\btheta + H(p(\btheta|\data) ,p(\btheta))} \\ \nonumber
  &= \log \int p(\btheta|\data) p^{M}(\btheta) d\btheta
   \\ \nonumber
  &- \log \int p(\btheta|\data) p^{M-1}(\btheta) d\btheta + H(p(\btheta|\data) ,p(\btheta)) \\\nonumber 
  & = \cblue{\log S_M- \log S_{M-1} + H(p(\btheta|\data) ,p(\btheta))}, \nonumber
\end{align}
\noindent where we used Bayes' rule in the last step, \cblue{and defined $S_M = \int p(\btheta | \data) p^M(\btheta) d\btheta$. As we will discuss next, if
$\log S_M$ is convex with respect to $M$, then we can show that~\eqref{eq:kl_diff_m} is positive, thus completing the proof. 
% This is shown in Appendix \ref{sec:exp} for the general class of exponential family distributions, which includes the Gaussian case. Alternatively, a complementary proof is provided for the Gaussian case in Appendix \ref{ap:GaussianThm2}.
}
% }
% \noindent where $CE \triangleq - \int  \frac{p(\data|\btheta) p(\btheta)}{p(\data)}    \log(p(\btheta))    d\btheta =- \int  p(\btheta|\data)  \log(p(\btheta))    d\btheta >0$ is the cross-entropy of the true posterior $p$. 

\cblue{\paragraph*{Convexity proof for $\log S_M$}
To demonstrate the convexity of $\log S_M$ note that it can be rewritten as $S_M=\log \mathbb{E}_p[a^M]$ for $a=p(\btheta)$, and $p=p(\btheta | \data)$. Furthermore, also note that $a$ is non-negative. Now we can show that functions of this type, $f(M) = \log \mathbb{E}_p[a^M]$, $a\geqslant 0$, are convex: 
\begin{align*}
    f(tM_1 + &(1-t)M_2) \mathop{=}^{(a)} \log \mathbb{E}_p[v^tu^{t-1}] \\
    &\mathop{\leqslant}^{(b)} \log \left(\mathbb{E}_p\left[\left|v^t\right|^\frac{1}{t}\right]\right)^t \left(\mathbb{E}_p\left[\left|u^{1-t}\right|^\frac{1}{(1-t)}\right]\right)^{1-t}  \\
    &\mathop{=}^{(c)}\log \mathbb{E}_p[v]^t \mathbb{E}_p[u]^{1-t} \\
    &= t\log \mathbb{E}_p[v] + (1-t) \log\mathbb{E}_p[u] \\
    &= tf(M_1) + (1-t)f(M_2).
\end{align*}
In equality $(a)$ we used $v = a^{M_1}$ and $u=a^{M_2}$. Inequality $(b)$ results from the application of the H\"older's inequality for Lebesgue measures, where $\mathbb{E}[XY] \leqslant \mathbb{E}[|X|^z]^\frac{1}{z} \mathbb{E}[|Y|^q]^\frac{1}{q}$ and we made $t=1/z$ and $(1-t)=1/q$, for $z,q\in[1,\infty)$. Equality $(c)$ follows from the fact that the quantities $v$ and $u$ are always positive, concluding the convexity proof for $\log S_M$. 
%since the sum of a linear and convex terms (terms 1 and 2 in~\eqref{eq:logSM}) results in a convex function.
}

\cblue{\paragraph*{Positiveness of~\eqref{eq:kl_diff_m}}
A direct consequence of the convexity of $\log S_M$, is that the difference $\log S_{M} - \log S_{M-1}$ can be bounded as $\log S_{M} - \log S_{M-1} \geq \log S_{1} - \log S_{0}$, with equality if $M=1$.
Similarly, we can use the result to show that 
$\mathrm{KL}_{M+1} - \mathrm{KL}_{M} \geq \mathrm{KL}_{2} - \mathrm{KL}_{1}$ is bounded when $M=1$.}
These results are then used to finally show the positiveness of \cblue{ \eqref{eq:kl_diff_m}}.
%Back to equation \ref{eq:kl_diff_m}, 
Based on the convexity of $\log S_M$ and Jensen's inequality, we can show that
\begin{align}\label{eq:KL_convex}
&\log S_M - \log S_{M-1} + H(p(\btheta|\data) ,p(\btheta)) \\ \nonumber
%&=\log \int p(\btheta|D) p^M(\btheta) d \btheta-\log \int p(\btheta | D) p^{M-1}(\btheta) d \btheta\\ \nonumber 
%&-\int p(\btheta | D) \log p(\btheta) d \btheta.\\ \nonumber
&=\log E_{p(\btheta |\data)}\left[p^M(\btheta)\right]-\log E_{p(\btheta|\data)}\left[p^{M-1}(\btheta)\right] d \btheta \\ \nonumber
&-E_{p(\btheta |\data)}\left[\log p(\btheta)\right] \\ \nonumber
&\geqslant_{(M=1)} \log E_{p(\btheta|\data)} \left[p(\btheta)\right]-E_{p(\btheta|\data)}\left[\log p(\btheta)\right]\\ \nonumber
&\geqslant E_{p(\btheta|\data)}\left[\log p(\btheta)\right]-E_{p(\btheta|\data)}\left[\log p(\btheta)\right] = 0, 
\end{align}
\noindent which concludes \cblue{that $\mathrm{KL}_{M+1} \geq \mathrm{KL}_{M}$, thus proving Theorem \ref{theorem_M} under KL divergence and arbitrary distributions. The following subsection in the appendix shows the result under the Gaussian assumption}.

\cblue{\subsection{Proof under Gaussian distributions}\label{ap:GaussianThm2}}
%\peng{add the Gaussian}:
%If only thinking about the Gaussian situation.
\cblue{A similar, although more tedious, result can be obtained under the more restrictive case of Gaussian distributed posteriors. Given the widespread use of Gaussian models, we provide this complementary result, which uses the Gaussian product results \cite{bromiley2003products}. Similarly, we aim at showing that $S_M$ is convex for the Gaussian case, such that the reasoning in equation \eqref{eq:KL_convex} follows.} %  The above proof has been proved for the exponential family for KL divergence, but this section prove it from the Gaussian product theory perspective, could rise more interest.}

We are interested in the Gaussian assumption for both the prior $p(\btheta)=\mathcal{N}(\btheta_0,\mathbf{C}_0)$ and the posterior $p(\btheta|\data)=\mathcal{N}(\boldsymbol{\mu},
\mathbf{\Lambda}^{-1})$ distributions. \cblue{The product in $S_M$ can be further manipulated as}
\begin{align}
%p^M(\btheta)&=S_{0,M}  \mathcal{N}\left(\btheta_0, \frac{\mathbf{C}_0}{M}\right)\\
p(\btheta | \data) &p^M(\btheta)=\mathcal{N}\left(\boldsymbol{\mu}, \mathbf{\Lambda}^{-1}\right) S_{0,M} \mathcal{N}\left(\btheta_0, \frac{\mathbf{C}_{0}}{M}\right)  \\ \nonumber
&=S_{0,M} S_{1,M} \; \mathcal{N} \left((M\mathbf{C}_{0}^{-1}+\mathbf{\Lambda})(\mathbf{\Lambda}\btheta_0 +\mathbf{C}_{0}^{-1}{M}\boldsymbol{\mu}), \right.  \\ \nonumber
&\left. (M\mathbf{C}_{0}^{-1}+\mathbf{\Lambda})^{-1})\right),
\end{align}
\noindent where $S_{0,M}$ and $S_{1,M}$ are  scaling factors \cite{bromiley2003products}
\begin{align}
S_{0,M}&=\frac{1}{(2 \pi)^{\frac{M-1}{2}}} \sqrt{\left| \frac{\mathbf{C}_{0}}{ M}\left(\mathbf{C}_{0}\right)^{-M}\right|} 
\nonumber\\
&\cdot 
\exp [-\frac{1}{2}(M \btheta_0^{\top}\mathbf{C}_{0}^{-1}\btheta_0-\btheta_0^{\top} (\frac{\mathbf{C}_{0}}{M})^{-1} \btheta_0)]\\ \nonumber
&=\frac{1}{(2 \pi)^{\frac{d(M-1)}{2}}}\left| M \mathbf{C}_{0} ^{M-1}\right|^{-1/2} \\
S_{1,M}&=\frac{1}{\sqrt{(2 \pi)^d\left|\frac{\mathbf{C}_{0}}{M}+\mathbf{\Lambda}^{-1}\right|}} 
\\ \nonumber
&\cdot
\exp \left[-\frac{1}{2}(\btheta_0-\boldsymbol{\mu})^{\top}\left(\frac{\mathbf{C}_{0}}{M}+\mathbf{\Lambda}^{-1}\right)^{-1}(\btheta_0-\boldsymbol{\mu})\right],
\end{align}
\noindent such that computing the integral $S_M$ reduces to a product of those scaling factors
\begin{align}
S_M&=S_{0,M} S_{1,M} \\ \nonumber
&=(2 \pi)^{-\frac{d M}{2}} \left|\frac{\mathbf{C}_{0}}{M}+\mathbf{\Lambda}^{-1}\right|^{-\frac{1}{2}} \left|M \mathbf{C}_{0}^{M-1}\right|^{-\frac{1}{2}} 
\\  \nonumber
&\exp \left[-\frac{1}{2}(\btheta_0-\boldsymbol{\mu})^{\top}\left(\frac{\mathbf{C}_{0}}{M}+\mathbf{\Lambda}^{-1}\right)^{-1}(\btheta_0-\boldsymbol{\mu})\right]
\\ \nonumber
&=(2 \pi)^{-\frac{d M}{2}}\left|\mathbf{C}_{0}^{M}+M \mathbf{C}_{0}^{M-1} \mathbf{\Lambda}^{-1}\right|^{-\frac{1}{2}} 
\\ \nonumber
&\exp \left[-\frac{1}{2}(\btheta_0-\boldsymbol{\mu})^{\top}\left(\frac{\mathbf{C}_{0}}{M}+\mathbf{\Lambda}^{-1}\right)^{-1}(\btheta_0-\boldsymbol{\mu})\right] \\ \nonumber
&=(2 \pi)^{-\frac{d}{2}}\left|(2 \pi)^d \mathbf{C}_{0}\right|^{-\frac{M-1}{2}}  \left|\mathbf{C}_{0} + M \mathbf{\Lambda}^{-1}\right|^{-\frac{1}{2}} 
\\ \nonumber
&\exp \left[-\frac{1}{2}(\btheta_0-\boldsymbol{\mu})^{\top}\left(\frac{\mathbf{C}_{0}}{M}+\mathbf{\Lambda}^{-1}\right)^{-1}(\btheta_0-\boldsymbol{\mu})\right] ,
\end{align}
%
% \pau{is the above particular of the scalar case?}
% \peng{reply: yes, agree, I will try}
\noindent and 
\begin{align}
    \log S_M &= -\frac{M-1}{2} \log \left|2 (\pi)^d \mathbf{C}_{0}\right| 
    % \\ \nonumber
    -\frac{1}{2} \log \left|\mathbf{C}_{0}+M \mathbf{\Lambda}^{-1}\right| \\ \nonumber
    &-\frac{1}{2}(\btheta_0-\boldsymbol{\mu})^{\top}\left(\frac{\mathbf{C}_{0}}{M}+\mathbf{\Lambda}^{-1}\right)^{-1}(\btheta_0-\boldsymbol{\mu})+{\rm const} \;.
\end{align}

In order to show that \eqref{eq:kl_m} is a positive quantity, we will first show that $\log S_M-\log S_{M-1}$ is bounded. 
% \pau{explain what we will do next, why are we focusing on log S now? what are we trying to accomplish?}
The objective is to show the $\log S_M-\log S_{M-1}$ increases with $M$, that is to say that $\log S_M$ is a convex function of $M$, in which case the boundness would follow with its minimum value being $\log S_1-\log S_0$ when $M=1$. The following derivation proves the convexity of $\log S_M$.
Using basic matrix algebra results, the first derivative of $\log S_M$ is
\begin{align}\label{eq:1d}
&\frac{\partial \log S_M}{\partial M} =-\frac{1}{2} \log \left|(2 \pi)^d \mathbf{C}_{0}\right|-\frac{1}{2} \textrm{Tr}((\mathbf{C}_{0}+M\mathbf{\Lambda}^{-1})^{-1} \mathbf{\Lambda}^{-1})\\ \nonumber
% \left| \mathbf{\Lambda}^{-1}(\mathbf{C}_{0}+M \mathbf{\Lambda}^{-1})^{-1}\right| \\ \nonumber
&-\frac{1}{2}(\btheta_0-\boldsymbol{\mu})^{\top}(\left(\mathbf{C}_{0}+M\mathbf{\Lambda}^{-1}\right)^{-1}\mathbf{C}_{0} \left(\mathbf{C}_{0}+M\mathbf{\Lambda}^{-1}\right)^{-1}) (\btheta_0-\boldsymbol{\mu}), %\\ \nonumber
% &=-\frac{\bm\Sigma_{p}+(\boldsymbol{\mu}_{0}-\boldsymbol{\mu})^{\top}(\boldsymbol{\mu}_{0}-\boldsymbol{\mu}) \bm\Sigma_{0}}{2\left(\bm\Sigma_{0}+M \bm\Sigma_{p}\right)}-\frac{1}{2} \log \left((2 \pi)^d \bm\Sigma_{0}\right)
\end{align}
and its second derivative is
\begin{align}\label{eq:2d}
&\frac{\partial^{2} \log S_M}{\partial M^{2}}=
 \frac{1}{2}\textrm{Tr}((\mathbf{C}_{0} \mathbf{\Lambda}+M)^{-2} + \\ \nonumber &\frac{1}{2}(\btheta_0-\boldsymbol{\mu})^{\top}[\left(\mathbf{C}_{0}+M\mathbf{\Lambda}^{-1}\right)^{-1}\mathbf{\Lambda}^{-1}\left(\mathbf{C}_{0}+M\mathbf{\Lambda}^{-1}\right)^{-1}](\btheta_0-\boldsymbol{\mu})
\end{align}

Since the first term in equation \eqref{eq:2d} is larger than $0$ and the second term is larger or equal than $0$, it follows that $\log S_M$ is a convex function and the discussion in equation \eqref{eq:KL_convex} holds for the Gaussian case. \cblue{While the result of the theorem for arbitrary distributions is mode general, the above result under Gaussian distributions is complementary and provides consistent results.}
\section{Proof of Theorem \ref{theorem_M} under $\chi^2$ divergence}\label{ap:chi2div}

The results in Theorem \ref{theorem_M} also holds when the $\chi^2$ divergence \cite{nielsen2013chi} is considered. 
Before showing the result let us first express the $\chi^2$ divergence as a function of the number of clients $M$:
% Similarly as done in Appendix \ref{ap:thm_M} for the KL, we first expand the expression of the $\chi^2$ divergence as
\begin{align}\label{eq:chi-square2}
    \chi^{2}_M&(p(\btheta|\data)\parallel\wt p(\btheta|\data)) = \int\frac{\left(p(\btheta|\data)-\wt p(\btheta|\data)\right)^{2}}{\wt p(\btheta|\data)}d\btheta \\ \nonumber
    &= \int\frac{p^2(\btheta|\data)}{\wt p(\btheta|\data)} d\btheta - 1 \\\nonumber
    &= \int \frac{\frac{p^2(\data|\btheta)p^2(\btheta)}{p^2(\data)}}{\frac{p(\data|\btheta)p^M(\btheta)}{p_M(\data)}} d\btheta - 1 \\\nonumber
    & = \int p(\data|\btheta)p^{2-M}(\btheta) \frac{p_M(\data)}{p^2(\data)} d\btheta -1 \\\nonumber
    & = \int p(\btheta|\data)p^{1-M}(\btheta) \frac{p_M(\data)}{p(\data)} d\btheta - 1 \\\nonumber
    & =  \int p(\btheta|\data)p^{M-1}(\btheta)d\btheta \int p(\btheta|\data)p^{1-M}(\btheta)d\btheta -1 \\
    & = \mathbb{E}_{p(\btheta|\data)}\left[p(\btheta)^{M-1}\right] \mathbb{E}_{p(\btheta|\data)}\left[p(\btheta)^{1-M}\right]\,. \nonumber
\end{align}

Now, we focus on proving that $\chi^2_{M+1} - \chi^2_M \geq 0$ for all $M\geq 1$, where we omit the arguments in~\eqref{eq:chi-square2} for brevity. For this, first note that $\chi^2_M$ is $\log$-convex, that is
$$\log\chi^2_M = \log\mathbb{E}_{p(\btheta|\data)}\left[p(\btheta)^{M-1}\right] + \log \mathbb{E}_{p(\btheta|\data)}\left[p(\btheta)^{1-M}\right]$$ 
is convex. The convexity of $\log\chi^2_M$ follows from the fact that $(i) $ both $\log \mathbb{E}_{p(\btheta|\data)}\left[p(\btheta)^{M-1}\right]$ and $\log \left[p(\btheta)^{1-M}\right]$ fall in the category of functions $f(M) =\log \mathbb{E}_p[a^M]$, $a\geqslant 0$, for which we proved convexity in Appendix~\ref{ap:thm_M}; and $(ii)$ the fact that the sum of convex functions is also convex~\cite{boyd2004convex}. 

Another necessary result for our proof relies on the fact that the $\log$ function is monotonically increasing for non-negative arguments such as the $\chi^2$ divergence. Thus, 
$$
\log \chi^2_{M+1} - \log \chi^2_M  \geqslant 0 \iff \chi^2_{M+1} - \chi^2_M \geqslant 0 \,.
$$
The above fact implies that if one can show that $\log \chi^2_{M+1} - \log \chi^2_M  \geq 0$, the proof is complete. 
Now, we can apply these results to complete our proof as follows:
%
% We aim at showing that $\chi^2_{M+1} > \chi^2_M$ , for which we
% will show that the following quantity is positive for $M\geq 1$:
\begin{align}
    % &\chi^2_{M+1} - \chi^2_M \\ \nonumber
    % & \Rightarrow  
    \log  \chi^2_{M+1} - \log \chi^2_M &\mathop{=}^{(a)} \log \left(\mathbb{E}_p\left[a^M\right]\mathbb{E}_p\left[a^{-M}\right]\right) \\
    &\quad - \log \left(\mathbb{E}_p\left[a^{M-1}\right]\mathbb{E}_p\left[a^{-(M-1)}\right]\right) \nonumber\\ \nonumber
    & {\mathop{\geqslant}^{(b)}}_{(M=1)} \log \left(\mathbb{E}_p[a]\mathbb{E}_p[a^{-1}]\right) \\ \nonumber
    & = \log\left( \mathbb{E}_p\left[\left(a^{\frac{1}{2}}\right)^2\right] \mathbb{E}_p\left[\left(a^{-\frac{1}{2}}\right)^2\right]\right) \\ \nonumber
    & \mathop{\geqslant}^{(c)} \log\left(\left|\mathbb{E}_p\left[a^{\frac{1}{2}} a^{-\frac{1}{2}}\right]\right|^2\right) = 0
\end{align}
where in Equality $(a)$ we defined $a=p(\btheta)$ and $p = p(\theta|\data)$ for brevity of notation. In Inequality $(b)$ we exploited the convexity of $\log\chi^2_M$ in the same way done in Appendix~\ref{ap:thm_M}, that is, $\log\chi^2_M -\log\chi^2_{M-1} \geqslant \log\chi^2_1 -\log\chi^2_{0}$. In the last step we applied the Cauchy-Schwarz to obtain Inequality $(c)$. 
Finally, the result above implies that 
$$
\chi^2_{M+1}\geqslant \chi^2_{M},\quad \forall M\geqslant 1
$$
which concludes the proof.

\balance
\bibliographystyle{IEEEtran}
% \bibliography{IEEEabrv,biblio,Federated_Learning}
\bibliography{IEEEabrv,main}

\end{document}

%% file: latex_figures/kl_copy.tex
% \documentclass[tikz]{standalone}

% \usepackage{pgfplots,siunitx,mathtools}

% \pgfplotsset{compat=newest}

% \begin{document}

\begin{tikzpicture}
  \begin{axis}[
      ymode=log,
      % domain=0:81,
      % ymax=130,
      enlargelimits=false,
      ylabel=$\mathrm{KL}$,
    %   xlabel= Prior Variance $q_0$,
    xlabel= $q_0$,
      xtick={1,2,3,4,5,6,7,8,9,10},
      xticklabels={0.1,0.2,0.5,1,2,3,4,9,16,25},
      % ,36,64,81},
      grid=both,
      width=8cm,
      height=4.5cm,
      decoration={name=none},
    ]
    \addplot [ thick, smooth, magenta!85!black] coordinates {
    (1, 62270.328125)
(2, 18016.845703125)
(3, 3147.502685546875)
(4, 842.195068359375)
(5, 206.84666442871094)
(6, 92.39319610595703)
(7, 51.882320404052734)
(8, 10.59694766998291)
(9, 3.3211252689361572)
(10, 1.4120171070098877)
    }node[pos=0.95, anchor=east] {};
    \addplot [ thick, smooth, cyan!85!black] coordinates {
    (1, 83780.4453125)
(2, 24086.376953125)
(3, 4439.68603515625)
(4, 1101.6607666015625)
(5, 290.29736328125)
(6, 126.54830932617188)
(7, 71.86840057373047)
(8, 14.684944152832031)
(9, 4.659680366516113)
(10, 1.8206592798233032)
    }node[pos=0.95, anchor=east] {} ;
    \legend{$M=6$,$M=26$}
   \end{axis}
\end{tikzpicture}

% \end{document}

%% file: latex_figures/kl_client.tex
% \documentclass[tikz]{standalone}

% \usepackage{pgfplots,siunitx,mathtools}

% \pgfplotsset{compat=newest}

% \begin{document}

\begin{tikzpicture}
  \begin{axis}[
      ymode=log,
    %   domain=1:180,
    %   ymax=400,
      enlargelimits=false,
      ylabel=$\mathrm{KL}$,
      xlabel= $M$,
      grid=both,
      width=8cm,
      height=4.5cm,
      decoration={name=none},
      legend style={at={(1,0)},anchor=south east}
    ]
    \addplot [thick, smooth, magenta!85!black] coordinates {
    (2.0, 0.20044004917144775)
(6.0, 5.026119709014893)
(10.0, 16.1180477142334)
(14.0, 33.770240783691406)
(18.0, 57.843528747558594)
(22.0, 87.82818603515625)
(26.0, 124.98237609863281)
(30.0, 165.97596740722656)
(34.0, 213.87767028808594)
(38.0, 270.43927001953125)
(42.0, 328.4809875488281)
(46.0, 397.11773681640625)
(50.0, 470.9361267089844)
    }node[pos=0.95, anchor=east] {};
    \addplot [thick, smooth, cyan!85!black] coordinates {
    (2.0, 0.02240588329732418)
(6.0, 0.5611779093742371)
(10.0, 1.81044340133667)
(14.0, 3.765305519104004)
(18.0, 6.459330081939697)
(22.0, 9.917047500610352)
(26.0, 14.115721702575684)
(30.0, 18.897422790527344)
(34.0, 24.380002975463867)
(38.0, 30.77215003967285)
(42.0, 37.530616760253906)
(46.0, 45.26709747314453)
(50.0, 54.31648635864258)
    }node[pos=0.95, anchor=east] {};
    \legend{$q_0=1$,$q_0=3$}
   \end{axis}
\end{tikzpicture} 

% \end{document}

%% file: latex_figures/req_q_test_copy.tex
% \documentclass[tikz]{standalone}

% \usepackage{pgfplots,siunitx,mathtools}

% \pgfplotsset{compat=newest}

% \begin{document}

\begin{tikzpicture}
  \begin{axis}[
      ymode=log,
    %   domain=0:180,
    %   ymin=0,
    %   ymax=0.04,
      enlargelimits=false,
      ylabel=MSE of Test,
      xlabel= $q_0$,
      xtick={1,2,3,4,5,6,7,8,9,10},
      xticklabels={0.1,0.2,0.5, 1, 2,3,4,9,16,25},
      % ,36,64,81},
      grid=both,
      width=9cm,
      height=5.5cm,
      decoration={name=none},
    ]
    \addplot [thick, smooth, magenta!85!black] coordinates {
    (1, 405.9111328125)
(2, 145.3505401611328)
(3, 37.066749572753906)
(4, 21.139690399169922)
(5, 18.0627498626709)
(6, 18.000883102416992)
(7, 16.29986000061035)
(8, 16.54392433166504)
(9, 16.815330505371094)
(10, 16.889911651611328)
    }node[pos=0.95, anchor=east] {};
    \addplot [thick, smooth, cyan!70!black] coordinates {
        (1.0, 16.540281295776367)
(2.0, 16.877775192260742)
(3.0, 15.860363960266113)
(4.0, 16.185468673706055)
(5.0, 16.219221115112305)
(6.0, 16.360149383544922)
(7.0, 17.164709091186523)
(8.0, 16.36081886291504)
(9.0, 16.094436645507812)
(10.0, 16.15765380859375)
      } node[pos=0.95, anchor=east] {};
%       \addplot [thick,mark = *, smooth, green!70!black] coordinates {
%         (1.0, 16.386266708374023)
% (2.0, 16.15638542175293)
% (3.0, 16.071332931518555)
% (4.0, 16.20641326904297)
% (5.0, 16.185745239257812)
% (6.0, 16.10102081298828)
% (7.0, 16.1400089263916)
% (8.0, 16.205007553100586)
% (9.0, 16.30921173095703)
% (10.0, 16.14332389831543)
%       } node[pos=0.95, anchor=east];
     \addplot [thick, smooth,dashed, magenta!70!black] coordinates {
        (1, 548.5079956054688)
(2, 169.12965393066406)
(3, 44.93806076049805)
(4, 22.247940063476562)
(5, 19.165096282958984)
(6, 17.917631149291992)
(7, 16.21212387084961)
(8, 17.32486343383789)
(9, 15.832789421081543)
(10, 15.521480560302734)
      } node[pos=0.95, anchor=east] {};
      \addplot [thick, smooth,dashed, cyan!70!black] coordinates {
        (1.0, 16.355663299560547)
(2.0, 17.555755615234375)
(3.0, 15.747491836547852)
(4.0, 15.387202262878418)
(5.0, 16.959062576293945)
(6.0, 16.160202026367188)
(7.0, 15.695343017578125)
(8.0, 15.911337852478027)
(9.0, 15.87015438079834)
(10.0, 15.720735549926758)
      } node[pos=0.95, anchor=east] {};
%       \addplot [thick, smooth,dashed, green!70!black] coordinates {
%         (1.0, 16.887346267700195)
% (2.0, 16.492740631103516)
% (3.0, 16.083637237548828)
% (4.0, 16.255054473876953)
% (5.0, 16.23600959777832)
% (6.0, 16.080820083618164)
% (7.0, 16.366357803344727)
% (8.0, 16.287105560302734)
% (9.0, 16.000730514526367)
% (10.0, 16.20700454711914)
%       } node[pos=0.95, anchor=east];
\legend{CIP $M=6$, CIL $M=6$,CIP $M=26$, CIL $M=26$}; %, VIC $M=26$
   \end{axis}
\end{tikzpicture} 

% \end{document}

%% file: latex_figures/reg_client_test.tex
% \documentclass[tikz]{standalone}

% \usepackage{pgfplots,siunitx,mathtools}

% \pgfplotsset{compat=newest}

% \begin{document}

\begin{tikzpicture}
  \begin{axis}[
    %   ymode=log,
      domain=1:51,
      ymax=20,
      enlargelimits=false,
      ylabel=MSE of Test,
      xlabel= $M$,
      grid=both,
      width=9cm,
      height=5.5cm,
      decoration={name=none},
      legend style={at={(0,1)},anchor=north west}
    ]
    \addplot [thick, smooth, magenta!85!black] coordinates {
    (2.0, 16.23455047607422)
(6.0, 16.418201446533203)
(10.0, 16.299339294433594)
(14.0, 16.477399826049805)
(18.0, 16.78753662109375)
(22.0, 16.8923397064209)
(26.0, 17.344886779785156)
(30.0, 17.50790786743164)
(34.0, 17.924619674682617)
(38.0, 18.58188819885254)
(42.0, 18.907922744750977)
(46.0, 19.261962890625)
(50.0, 19.879724502563477)
    }node[pos=0.95, anchor=east] {};
    \addplot [thick, smooth, cyan!70!black] coordinates {
        (2.0, 16.22841453552246)
(6.0, 16.360355377197266)
(10.0, 16.138107299804688)
(14.0, 16.206796646118164)
(18.0, 16.272436141967773)
(22.0, 16.16705894470215)
(26.0, 16.280431747436523)
(30.0, 16.090641021728516)
(34.0, 16.183456420898438)
(38.0, 16.249225616455078)
(42.0, 16.28784942626953)
(46.0, 16.046762466430664)
(50.0, 16.11885643005371)
      } node[pos=0.95, anchor=east] {};
%       \addplot [thick, smooth, green!70!black] coordinates {
%         (2.0, 16.231836318969727)
% (6.0, 16.386266708374023)
% (10.0, 16.231914520263672)
% (14.0, 16.3582706451416)
% (18.0, 16.568313598632812)
% (22.0, 16.573501586914062)
% (26.0, 16.887346267700195)
% (30.0, 16.917524337768555)
% (34.0, 17.16686248779297)
% (38.0, 17.558359146118164)
% (42.0, 17.724746704101562)
% (46.0, 17.917560577392578)
% (50.0, 18.213857650756836)
%       } node[pos=0.95, anchor=east];
    \addplot [thick, smooth,dashed, magenta!70!black] coordinates {
        (2.0, 16.22533416748047)
(6.0, 16.076948165893555)
(10.0, 16.063566207885742)
(14.0, 16.197534561157227)
(18.0, 16.38532257080078)
(22.0, 16.289003372192383)
(26.0, 16.132247924804688)
(30.0, 16.213294982910156)
(34.0, 16.337968826293945)
(38.0, 16.48887825012207)
(42.0, 16.52815055847168)
(46.0, 16.62458610534668)
(50.0, 16.529708862304688)
      } node[pos=0.95, anchor=east] {};
      \addplot [thick, smooth,dashed, cyan!70!black] coordinates {
        (2.0, 16.22494888305664)
(6.0, 16.06818389892578)
(10.0, 16.04627799987793)
(14.0, 16.164636611938477)
(18.0, 16.326311111450195)
(22.0, 16.20682716369629)
(26.0, 16.028976440429688)
(30.0, 16.053985595703125)
(34.0, 16.116470336914062)
(38.0, 16.258472442626953)
(42.0, 16.206928253173828)
(46.0, 16.25568389892578)
(50.0, 16.130817413330078)
      } node[pos=0.95, anchor=east] {};
%       \addplot [thick, smooth,dashed, green!70!black] coordinates {
%         (2.0, 16.225576400756836)
% (6.0, 16.071332931518555)
% (10.0, 16.057506561279297)
% (14.0, 16.1795711517334)
% (18.0, 16.36699867248535)
% (22.0, 16.259323120117188)
% (26.0, 16.083637237548828)
% (30.0, 16.144960403442383)
% (34.0, 16.23536491394043)
% (38.0, 16.37611961364746)
% (42.0, 16.399229049682617)
% (46.0, 16.464841842651367)
% (50.0, 16.387828826904297)
%       } node[pos=0.95, anchor=east];
    \legend{CIP $q_0=1$,CIL $q_0=1$,CIP $q_0=3$,CIL $q_0=3$} %,VIC $q_0=1$,VIC $q_0=3$
   \end{axis}
\end{tikzpicture} 

% \end{document}

%% file: latex_figures/clf_pro_kl_q.tex
% \documentclass[tikz]{standalone}

% \usepackage{pgfplots,siunitx,mathtools}

% \pgfplotsset{compat=newest}

% \begin{document}

\begin{tikzpicture}
  \begin{axis}[
      ymode=log,
    %   domain=1:33,
    %   ymax=0.5,
      enlargelimits=false,
      ylabel=$\mathrm{KL}$,
      xlabel= $P_1$,
      grid=both,
      width=8cm,
      height=4.5cm,
      decoration={name=none},
      legend style={at={(0.99,0.09)},anchor=south east}
    ]
    \addplot [ thick, smooth, magenta!85!black] coordinates {
        (0.024390243902439025, 1.769158843058432)
(0.04878048780487805, 1.2084791060688598)
(0.07317073170731707, 0.8806949751355362)
(0.0975609756097561, 0.635036224349562)
(0.12195121951219512, 0.469916182040783)
(0.14634146341463414, 0.33753023361794243)
(0.17073170731707318, 0.23553538551517203)
(0.1951219512195122, 0.16252141063146378)
(0.21951219512195122, 0.11865075925612699)
(0.24390243902439024, 0.08517499798731626)
(0.2682926829268293, 0.062429175311709946)
(0.2926829268292683, 0.047028477754790264)
(0.3170731707317073, 0.03552725714890411)
(0.34146341463414637, 0.026204465720612455)
(0.36585365853658536, 0.018778033935006665)
(0.3902439024390244, 0.012671259293574882)
(0.4146341463414634, 0.007678989112564052)
(0.43902439024390244, 0.00389747307372915)
(0.4634146341463415, 0.0013861240669430864)
(0.4878048780487805, 0.00015183799211106004)
(0.5121951219512195, 0.00015082450312451444)
(0.5365853658536586, 0.0013730116958708292)
(0.5609756097560976, 0.003915979209475125)
(0.5853658536585366, 0.008032885330258586)
(0.6097560975609756, 0.014324566536573456)
(0.6341463414634146, 0.02346394699136167)
(0.6585365853658537, 0.03631564218489927)
(0.6829268292682927, 0.052283653965151444)
(0.7073170731707317, 0.07773256790677911)
(0.7317073170731707, 0.10866995060757714)
(0.7560975609756098, 0.14813535599748856)
(0.7804878048780488, 0.19492388884337647)
(0.8048780487804879, 0.2438668965964245)
(0.8292682926829268, 0.2921059237005871)
(0.8536585365853658, 0.33832349454913113)
(0.8780487804878049, 0.3770373212993625)
(0.9024390243902439, 0.4266824771216467)
(0.926829268292683, 0.565599559925802)
(0.9512195121951219, 0.9149154879244662)
(0.975609756097561, 1.6081048635941841)
    }node[pos=0.95, anchor=east] {};
    \addplot [ thick, smooth, cyan!85!black] coordinates {
       (0.024390243902439025, 3.8997533157149586)
(0.04878048780487805, 2.638913676591632)
(0.07317073170731707, 1.784949358574845)
(0.0975609756097561, 1.277253050573509)
(0.12195121951219512, 0.8886515861164938)
(0.14634146341463414, 0.5808159953443414)
(0.17073170731707318, 0.40249796958971396)
(0.1951219512195122, 0.28560533100335717)
(0.21951219512195122, 0.21252177538721045)
(0.24390243902439024, 0.17322145029771147)
(0.2682926829268293, 0.13200799819039175)
(0.2926829268292683, 0.10930768756441292)
(0.3170731707317073, 0.08208027788216105)
(0.34146341463414637, 0.059974361103670805)
(0.36585365853658536, 0.04448823886045813)
(0.3902439024390244, 0.02909636282770245)
(0.4146341463414634, 0.017708416762254646)
(0.43902439024390244, 0.008781882049390472)
(0.4634146341463415, 0.0031090067052287063)
(0.4878048780487805, 0.00030509135436759534)
(0.5121951219512195, 0.0003062035728529544)
(0.5365853658536586, 0.0026740773844858567)
(0.5609756097560976, 0.00763125194384854)
(0.5853658536585366, 0.017567878369775504)
(0.6097560975609756, 0.031352426395695604)
(0.6341463414634146, 0.05483349488828601)
(0.6585365853658537, 0.08304321348955866)
(0.6829268292682927, 0.1167978604176099)
(0.7073170731707317, 0.16604012244788036)
(0.7317073170731707, 0.24823022842127287)
(0.7560975609756098, 0.3374850796825608)
(0.7804878048780488, 0.44079603554099234)
(0.8048780487804879, 0.5690621153470392)
(0.8292682926829268, 0.6790026194500913)
(0.8536585365853658, 0.8072127082186339)
(0.8780487804878049, 0.9069226334320659)
(0.9024390243902439, 0.9777413163804219)
(0.926829268292683, 1.1643082958311737)
(0.9512195121951219, 1.9072343441617838)
(0.975609756097561, 3.245596329382494)
    }node[pos=0.95, anchor=east] {};
    \legend{$M=6$,$M=12$}
   \end{axis}
\end{tikzpicture}

% \end{document}

%% file: latex_figures/clf_pro_kl_client.tex
% \documentclass[tikz]{standalone}

% \usepackage{pgfplots,siunitx,mathtools}

% \pgfplotsset{compat=newest}

% \begin{document}

\begin{tikzpicture}
  \begin{axis}[
      ymode=log,
    %   domain=1:33,
    %   ymax=0.5,
      enlargelimits=false,
      ylabel=$\mathrm{KL}$,
      xlabel= $M$,
      xtick={1,2,3,4,5,6,7,8},
      xticklabels={3, 6, 9, 12, 15, 18, 21, 24},
      grid=both,
      width=8cm,
      height=4.5cm,
      decoration={name=none},
    ]
    \addplot [ thick, smooth, magenta!85!black] coordinates {
        (1, 0.21066455740666498)
(2, 0.6117081218781895)
(3, 0.9557904491160777)
(4, 1.2266962864917903)
(5, 1.3673500658084823)
(6, 1.437918841605807)
(7, 1.4369407357529034)
(8, 1.4071897551860575)
    }node[pos=0.95, anchor=east] {};
    \addplot [ thick, smooth, cyan!85!black] coordinates {
       (1, 0.0038183199130256583)
(2, 0.010434046262900162)
(3, 0.017850087820631175)
(4, 0.024838308606068643)
(5, 0.030572194221411563)
(6, 0.035621465079733225)
(7, 0.0382580053619478)
(8, 0.04198566237981956)
    }node[pos=0.95, anchor=east] {};
    \legend{$P_1=0.1$,$P_1=0.4$}
   \end{axis}
\end{tikzpicture}

% \end{document}

%% file: latex_figures/clf_pro_q_acc.tex
% \documentclass[tikz]{standalone}

% \usepackage{pgfplots,siunitx,mathtools}

% \pgfplotsset{compat=newest}

% \begin{document}

\begin{tikzpicture}
  \begin{axis}[
    %   ymode=log,
    %   domain=1:33,
    %   ymax=0.9,
      enlargelimits=false,
      ylabel= Accuracy of Test,
      xlabel= $P_1$,
      grid=both,
      width=9cm,
      height=5.5cm,
      decoration={name=none},
      legend style={at={(0.99,0.01)},anchor=south east}
    ]
%     \addplot [thick, smooth, green!70!black] coordinates {
%         (0.024390243902439025, 0.7450000000000001)
% (0.04878048780487805, 0.7900000000000003)
% (0.07317073170731707, 0.7999999999999998)
% (0.0975609756097561, 0.8349999999999997)
% (0.12195121951219512, 0.8550000000000004)
% (0.14634146341463414, 0.8649999999999999)
% (0.17073170731707318, 0.875)
% (0.1951219512195122, 0.905)
% (0.21951219512195122, 0.9099999999999998)
% (0.24390243902439024, 0.9200000000000002)
% (0.2682926829268293, 0.9200000000000002)
% (0.2926829268292683, 0.925)
% (0.3170731707317073, 0.9300000000000002)
% (0.34146341463414637, 0.9300000000000002)
% (0.36585365853658536, 0.9350000000000004)
% (0.3902439024390244, 0.9350000000000004)
% (0.4146341463414634, 0.9399999999999996)
% (0.43902439024390244, 0.9399999999999996)
% (0.4634146341463415, 0.95)
% (0.4878048780487805, 0.95)
% (0.5121951219512195, 0.9549999999999998)
% (0.5365853658536586, 0.9549999999999998)
% (0.5609756097560976, 0.95)
% (0.5853658536585366, 0.95)
% (0.6097560975609756, 0.95)
% (0.6341463414634146, 0.96)
% (0.6585365853658537, 0.9549999999999998)
% (0.6829268292682927, 0.9549999999999998)
% (0.7073170731707317, 0.96)
% (0.7317073170731707, 0.9549999999999998)
% (0.7560975609756098, 0.95)
% (0.7804878048780488, 0.96)
% (0.8048780487804879, 0.96)
% (0.8292682926829268, 0.9549999999999998)
% (0.8536585365853658, 0.9549999999999998)
% (0.8780487804878049, 0.9549999999999998)
% (0.9024390243902439, 0.95)
% (0.926829268292683, 0.925)
% (0.9512195121951219, 0.9000000000000005)
% (0.975609756097561, 0.8499999999999995)
%       }  node[pos=0.95, anchor=east] {};
    \addplot [thick, smooth, magenta!85!black] coordinates {
    (0.024390243902439025, 0.7547999999999999)
(0.04878048780487805, 0.7943000000000001)
(0.07317073170731707, 0.8226499999999998)
(0.0975609756097561, 0.8478999999999999)
(0.12195121951219512, 0.86335)
(0.14634146341463414, 0.8791000000000001)
(0.17073170731707318, 0.8968999999999998)
(0.1951219512195122, 0.9096500000000001)
(0.21951219512195122, 0.9154000000000001)
(0.24390243902439024, 0.9205500000000001)
(0.2682926829268293, 0.9252000000000001)
(0.2926829268292683, 0.9281500000000003)
(0.3170731707317073, 0.9311000000000001)
(0.34146341463414637, 0.9341500000000003)
(0.36585365853658536, 0.9356000000000003)
(0.3902439024390244, 0.93845)
(0.4146341463414634, 0.9416499999999998)
(0.43902439024390244, 0.94475)
(0.4634146341463415, 0.9488999999999999)
(0.4878048780487805, 0.9516999999999998)
(0.5121951219512195, 0.95255)
(0.5365853658536586, 0.95155)
(0.5609756097560976, 0.9506)
(0.5853658536585366, 0.9512)
(0.6097560975609756, 0.9541499999999998)
(0.6341463414634146, 0.9561499999999998)
(0.6585365853658537, 0.955)
(0.6829268292682927, 0.9563499999999999)
(0.7073170731707317, 0.9566499999999999)
(0.7317073170731707, 0.9547999999999998)
(0.7560975609756098, 0.9544499999999999)
(0.7804878048780488, 0.9562499999999998)
(0.8048780487804879, 0.9581499999999998)
(0.8292682926829268, 0.9566999999999999)
(0.8536585365853658, 0.9548499999999999)
(0.8780487804878049, 0.9548499999999999)
(0.9024390243902439, 0.95345)
(0.926829268292683, 0.9422000000000001)
(0.9512195121951219, 0.9060000000000001)
(0.975609756097561, 0.8618000000000001)
    }node[pos=0.95, anchor=east] {};
    \addplot [thick, smooth, cyan!70!black] coordinates {
        (0.024390243902439025, 0.9349500000000003)
(0.04878048780487805, 0.9368500000000001)
(0.07317073170731707, 0.9390999999999999)
(0.0975609756097561, 0.9407499999999998)
(0.12195121951219512, 0.9422999999999996)
(0.14634146341463414, 0.9431499999999997)
(0.17073170731707318, 0.9448499999999997)
(0.1951219512195122, 0.9448499999999999)
(0.21951219512195122, 0.9459499999999998)
(0.24390243902439024, 0.9473)
(0.2682926829268293, 0.9483499999999998)
(0.2926829268292683, 0.9483)
(0.3170731707317073, 0.94955)
(0.34146341463414637, 0.95015)
(0.36585365853658536, 0.9507499999999998)
(0.3902439024390244, 0.9510999999999998)
(0.4146341463414634, 0.95155)
(0.43902439024390244, 0.9515)
(0.4634146341463415, 0.95175)
(0.4878048780487805, 0.9523499999999997)
(0.5121951219512195, 0.9521999999999998)
(0.5365853658536586, 0.9524000000000001)
(0.5609756097560976, 0.9522999999999999)
(0.5853658536585366, 0.9522999999999999)
(0.6097560975609756, 0.95215)
(0.6341463414634146, 0.9519499999999997)
(0.6585365853658537, 0.95215)
(0.6829268292682927, 0.95205)
(0.7073170731707317, 0.9516)
(0.7317073170731707, 0.95155)
(0.7560975609756098, 0.951)
(0.7804878048780488, 0.9514499999999999)
(0.8048780487804879, 0.95105)
(0.8292682926829268, 0.951)
(0.8536585365853658, 0.9514999999999999)
(0.8780487804878049, 0.9511999999999999)
(0.9024390243902439, 0.9516499999999999)
(0.926829268292683, 0.9534999999999999)
(0.9512195121951219, 0.9554499999999999)
(0.975609756097561, 0.9555999999999999)
      }  node[pos=0.95, anchor=east] {};
      \addplot [thick, smooth, dashed, magenta!70!black] coordinates {
        (0.024390243902439025, 0.7756500000000002)
(0.04878048780487805, 0.8134000000000001)
(0.07317073170731707, 0.8451499999999998)
(0.0975609756097561, 0.8636499999999999)
(0.12195121951219512, 0.8820500000000001)
(0.14634146341463414, 0.89865)
(0.17073170731707318, 0.9095)
(0.1951219512195122, 0.9155)
(0.21951219512195122, 0.9209)
(0.24390243902439024, 0.9245000000000003)
(0.2682926829268293, 0.9273000000000002)
(0.2926829268292683, 0.9305000000000001)
(0.3170731707317073, 0.9336000000000004)
(0.34146341463414637, 0.9364500000000001)
(0.36585365853658536, 0.9379)
(0.3902439024390244, 0.9419999999999996)
(0.4146341463414634, 0.9440999999999999)
(0.43902439024390244, 0.948)
(0.4634146341463415, 0.9492999999999999)
(0.4878048780487805, 0.95125)
(0.5121951219512195, 0.95135)
(0.5365853658536586, 0.9512)
(0.5609756097560976, 0.95195)
(0.5853658536585366, 0.9533499999999999)
(0.6097560975609756, 0.9539)
(0.6341463414634146, 0.9555499999999998)
(0.6585365853658537, 0.9547499999999998)
(0.6829268292682927, 0.9547999999999999)
(0.7073170731707317, 0.9541999999999997)
(0.7317073170731707, 0.9540499999999996)
(0.7560975609756098, 0.9541)
(0.7804878048780488, 0.9556499999999999)
(0.8048780487804879, 0.95695)
(0.8292682926829268, 0.9562999999999999)
(0.8536585365853658, 0.9555999999999999)
(0.8780487804878049, 0.95455)
(0.9024390243902439, 0.9533999999999998)
(0.926829268292683, 0.94635)
(0.9512195121951219, 0.9190999999999999)
(0.975609756097561, 0.8825999999999999)
      }  node[pos=0.95, anchor=east] {};
      \addplot [thick, smooth,dashed, cyan!70!black] coordinates {
        (0.024390243902439025, 0.9453499999999998)
(0.04878048780487805, 0.9464999999999999)
(0.07317073170731707, 0.948)
(0.0975609756097561, 0.9479999999999998)
(0.12195121951219512, 0.9491499999999999)
(0.14634146341463414, 0.9489500000000001)
(0.17073170731707318, 0.9502999999999997)
(0.1951219512195122, 0.9496999999999999)
(0.21951219512195122, 0.95015)
(0.24390243902439024, 0.9496000000000001)
(0.2682926829268293, 0.9501999999999999)
(0.2926829268292683, 0.9502999999999997)
(0.3170731707317073, 0.9500499999999998)
(0.34146341463414637, 0.9503499999999998)
(0.36585365853658536, 0.9513499999999999)
(0.3902439024390244, 0.9506499999999999)
(0.4146341463414634, 0.9503999999999998)
(0.43902439024390244, 0.9510999999999998)
(0.4634146341463415, 0.9516499999999999)
(0.4878048780487805, 0.9513999999999999)
(0.5121951219512195, 0.95125)
(0.5365853658536586, 0.9509499999999999)
(0.5609756097560976, 0.9513499999999998)
(0.5853658536585366, 0.9516999999999999)
(0.6097560975609756, 0.9516499999999998)
(0.6341463414634146, 0.9515499999999999)
(0.6585365853658537, 0.9508999999999997)
(0.6829268292682927, 0.9514499999999999)
(0.7073170731707317, 0.9515999999999998)
(0.7317073170731707, 0.9508500000000001)
(0.7560975609756098, 0.9513999999999997)
(0.7804878048780488, 0.9516499999999999)
(0.8048780487804879, 0.95155)
(0.8292682926829268, 0.9512499999999998)
(0.8536585365853658, 0.9505999999999999)
(0.8780487804878049, 0.9511999999999999)
(0.9024390243902439, 0.9517500000000001)
(0.926829268292683, 0.9516499999999999)
(0.9512195121951219, 0.9519999999999998)
(0.975609756097561, 0.9526499999999999)
      }  node[pos=0.95, anchor=east] {};
    \legend{CIP $M=6$,CIL $M=6$,CIP $M=16$,CIL $M=16$} 
    % #Global $M=6$,
   \end{axis}
\end{tikzpicture} 

% \end{document}

%% file: latex_figures/clf_pro_client_acc.tex
% \documentclass[tikz]{standalone}

% \usepackage{pgfplots,siunitx,mathtools}

% \pgfplotsset{compat=newest}

% \begin{document}

\begin{tikzpicture}
  \begin{axis}[
    %   ymode=log,
    %   domain=1:33,
    %   ymax=0.98,
    ymin=0.82,
      enlargelimits=false,
      ylabel= Accuracy of Test,
      xlabel= $M$,
      xtick={1,2,3,4,5,6,7,8},
      xticklabels={3, 6, 9, 12, 15, 18, 21, 24},
      grid=both,
      width=9cm,
      height=5.5cm,
      decoration={name=none},
      legend style={at={(0,0)},anchor=south west,fill=white, fill opacity=0.8}
    ]
%     \addplot [thick, smooth, green!85!black] coordinates {
%     (1, 0.84)
% (2, 0.84)
% (3, 0.84)
% (4, 0.84)
% (5, 0.84)
% (6, 0.84)
% (7, 0.84)
% (8, 0.84)
%     }node[pos=0.95, anchor=east] ;
    \addplot [thick, smooth, magenta!70!black] coordinates {
        (1, 0.8415)
(2, 0.8502500000000001)
(3, 0.8578749999999999)
(4, 0.8666499999999999)
(5, 0.878025)
(6, 0.887975)
(7, 0.8969)
(8, 0.90485)
      }  node[pos=0.95, anchor=east] {};
      \addplot [thick, smooth, cyan!70!black] coordinates {
        (1, 0.9302499999999999)
(2, 0.9410249999999999)
(3, 0.94625)
(4, 0.948075)
(5, 0.94865)
(6, 0.9497249999999999)
(7, 0.949375)
(8, 0.9494499999999999)
      }  node[pos=0.95, anchor=east] {};
%       \addplot [thick, smooth,dashed, green!70!black] coordinates {
%         (1, 0.9350000000000003)
% (2, 0.9350000000000003)
% (3, 0.9350000000000003)
% (4, 0.9350000000000003)
% (5, 0.9350000000000003)
% (6, 0.9350000000000003)
% (7, 0.9350000000000003)
% (8, 0.9350000000000003)
%       }  node[pos=0.95, anchor=east] {};
      \addplot [thick, smooth,dashed, magenta!70!black] coordinates {
       (1, 0.938125)
(2, 0.9401999999999998)
(3, 0.9420000000000001)
(4, 0.9425999999999999)
(5, 0.9443)
(6, 0.94505)
(7, 0.9455499999999999)
(8, 0.9461750000000001)
      }  node[pos=0.95, anchor=east] {};
      \addplot [thick, smooth,dashed, cyan!70!black] coordinates {
        (1, 0.9485499999999999)
(2, 0.9511749999999999)
(3, 0.9514249999999999)
(4, 0.95065)
(5, 0.9507249999999999)
(6, 0.9507749999999998)
(7, 0.950225)
(8, 0.9501249999999999)
      }  node[pos=0.95, anchor=east] {};
    \legend{CIP $P_1=0.1$,CIL $P_1=0.1$,CIP $P_1=0.4$,CIL $P_1=0.4$} %Global $P=0.1$,Global $P=0.4$,
   \end{axis}
\end{tikzpicture} 

% \end{document}

%% file: latex_figures/clf_kl_q_2.tex
% \documentclass[tikz]{standalone}

% \usepackage{pgfplots,siunitx,mathtools}

% \pgfplotsset{compat=newest}

% \begin{document}

\begin{tikzpicture}
  \begin{axis}[
      ymode=log,
    %   domain=1:33,
    %   ymax=440,
      enlargelimits=false,
      ylabel=$\mathrm{KL}$,
      xlabel= $q_0$,
      xtick={1,2,3,4,5,6,7},
      xticklabels={1,2,4,9,16,25,32}, %,1.2,1.6
      grid=both,
      width=8cm,
      height=4.5cm,
      decoration={name=none},
    ]
    \addplot [ thick, smooth, magenta!85!black] coordinates {
    (1, 1570.9734722512815)
% (2, 240.9111863412365)
% (3, 42.07478909900728)
(2, 14.206043352513598)
(3, 0.7437517947643301)
(4, 0.03325316106649438)
(5, 0.004384801005312511)
(6, 0.0008970989137310426)
(7, 0.00038262975841121263)
    }node[pos=0.95, anchor=east] {};
    \addplot [ thick, smooth, cyan!85!black] coordinates {
    (1, 6956.589314876248)
% (2, 401.7932978437857)
% (3, 58.706039665984086)
(2, 18.970634378768285)
(3, 0.963320563907331)
(4, 0.044173534540175294)
(5, 0.005004704476129973)
(6, 0.001194996903836909)
(7, 0.0004963539723064514)
    }node[pos=0.95, anchor=east] {};
    \legend{$M=6$,$M=16$}
   \end{axis}
\end{tikzpicture}

% \end{document}

%% file: latex_figures/clf_kl_client_2.tex
% \documentclass[tikz]{standalone}

% \usepackage{pgfplots,siunitx,mathtools}

% \pgfplotsset{compat=newest}

% \begin{document}

\begin{tikzpicture}
  \begin{axis}[
      ymode=log,
    %   domain=1:33,
    %   ymax=0.5,
      enlargelimits=false,
      ylabel=$\mathrm{KL}$,
      xlabel= $M$,
      grid=both,
      width=8cm,
      height=4.5cm,
      decoration={name=none},
      legend style={at={(0.99,0.95)},anchor=north east}
    ]
    \addplot [ thick, smooth, magenta!85!black] coordinates {
        (4, 0.175491756335515)
(6, 0.23352651968369997)
(8, 0.26764905690345414)
(10, 0.29958772668819905)
(12, 0.3202434629819152)
(14, 0.33932271283832877)
(16, 0.3551608169603398)
(18, 0.3372799714805751)
(20, 0.34256368745591176)
(22, 0.3560758852129965)
(24, 0.35858445921444115)
(26, 0.38319136839704837)
(28, 0.3507309499193013)
(30, 0.3739485119452667)
(32, 0.3646547560453953)
    }node[pos=0.95, anchor=east] {};
    \addplot [ thick, smooth, cyan!85!black] coordinates {
       (4, 8.63696706326494)
(6, 11.92460946184081)
(8, 14.028001370781867)
(10, 15.652155516648454)
(12, 16.88035682392413)
(14, 17.79275069345548)
(16, 18.353939159508624)
(18, 19.028940584921816)
(20, 19.53940466496568)
(22, 19.873135455633598)
(24, 20.37664160799848)
(26, 20.43773799473348)
(28, 20.85743709648924)
(30, 21.117945572148155)
(32, 21.14162851373283)
    }node[pos=0.95, anchor=east] {};
    \legend{$q_0=4$,$q_0=1.6$}
   \end{axis}
\end{tikzpicture}

% \end{document}

%% file: latex_figures/clf_q_acc_2.tex
% \documentclass[tikz]{standalone}

% \usepackage{pgfplots,siunitx,mathtools}

% \pgfplotsset{compat=newest}

% \begin{document}

\begin{tikzpicture}
  \begin{axis}[
    %   ymode=log,
      domain=1:33,
      enlargelimits=false,
      ylabel= Accuracy of Test/Global,
      xlabel= $q_0$,
      xtick={1,2,3,4,5,6,7},
      xticklabels={1,2,4,9,16,25,32}, %1.2,1.6,
      grid=both,
      width=9cm,
      height=5.5cm,
      decoration={name=none},
      legend style={at={(1,1)},anchor=north east}
    ]
%     \addplot [thick, smooth, green!70!black] coordinates {
%       (1, 0.5860000000000001)
% (2, 0.579)
% (3, 0.6466666666666667)
% (4, 0.6406666666666667)
% (5, 0.6846666666666665)
% (6, 0.7106666666666667)
% (7, 0.76)
% (8, 0.7843333333333333)
% (9, 0.7779999999999999)
%       }  node[pos=0.95, anchor=east];
    \addplot [thick, smooth, magenta!85!black] coordinates {
    (1, 0.8996865203761756)
% (2, 0.8953229398663698)
% (3, 0.866549088771311)
(2, 0.872688853671421)
(3, 0.8623809523809525)
(4, 0.8352835283528353)
(5, 0.8083670715249662)
(6, 0.8260297984224364)
(7, 0.8450641876936699)
    }node[pos=0.95, anchor=east] {};
    \addplot [thick, smooth, cyan!70!black] coordinates {
        (1, 1.0658307210031348)
% (2, 1.1102449888641424)
% (3, 1.0699588477366258)
(2, 0.9978869519281565)
(3, 0.8823809523809524)
(4, 0.8402340234023404)
(5, 0.8097165991902834)
(6, 0.8260297984224364)
(7, 0.8459495351925632)
      }  node[pos=0.95, anchor=east] {};
%       \addplot [thick, smooth, dashed, green!70!black] coordinates {
%         (1, 0.6323333333333333)
% (2, 0.6733333333333333)
% (3, 0.7169999999999999)
% (4, 0.7443333333333332)
% (5, 0.82725)
% (6, 0.86525)
% (7, 0.8767499999999999)
% (8, 0.8944999999999999)
% (9, 0.8827500000000001)
%       }  node[pos=0.95, anchor=east];
      \addplot [thick, smooth, dashed, magenta!70!black] coordinates {
        (1, 0.7945544554455447)
% (2, 0.7532163742690057)
% (3, 0.7126632595116411)
(2, 0.7366511145671332)
(3, 0.6771539206195547)
(4, 0.6312022900763358)
(5, 0.6864367816091952)
(6, 0.6582827406764961)
(7, 0.6732629727352683)
      }  node[pos=0.95, anchor=east] {};
      \addplot [thick, smooth,dashed, cyan!70!black] coordinates {
        (1, 0.8347772277227724)
% (2, 1.0479532163742689)
% (3, 0.9057353776263486)
(2, 0.8346293416277863)
(3, 0.7042594385285577)
(4, 0.6364503816793893)
(5, 0.6868965517241378)
(6, 0.6595836947094537)
(7, 0.6732629727352683)
      }  node[pos=0.95, anchor=east] {};
    \legend{CIP $M=6$,CIL $M=6$,CIP $M=16$,CIL $M=16$}
   \end{axis}
\end{tikzpicture} 

% \end{document}

 

%% file: latex_figures/clf_client_acc_2.tex
% \documentclass[tikz]{standalone}

% \usepackage{pgfplots,siunitx,mathtools}

% \pgfplotsset{compat=newest}

% \begin{document}

\begin{tikzpicture}
  \begin{axis}[
    %   ymode=log,
      domain=1:33,
      enlargelimits=false,
      ylabel= Accuracy of Test/Global ,
      xlabel= number of clients $M$,
      grid=both,
      width=9cm,
      height=5.5cm,
      decoration={name=none},
      legend style={at={(0,0)},anchor=south west,fill=white, fill opacity=.8}
    ]
%     \addplot [thick, smooth, green!85!black] coordinates {
%     (4, 0.9655833333333333)
% (6, 0.9654999999999998)
% (8, 0.9624999999999999)
% (10, 0.9644166666666665)
% (12, 0.9646666666666667)
% (14, 0.9630833333333332)
% (16, 0.9606666666666664)
% (18, 0.9628333333333332)
% (20, 0.9590833333333333)
% (22, 0.9645833333333332)
% (24, 0.9640000000000001)
% (26, 0.96375)
% (28, 0.9664166666666665)
% (30, 0.9628333333333332)
% (32, 0.9601666666666666)
%     }node[pos=0.95, anchor=east] ;
    \addplot [thick, smooth, magenta!70!black] coordinates {
        (4, 0.9381443298969073)
(6, 0.9075630252100839)
(8, 0.8653745416448402)
(10, 0.8223615464994775)
(12, 0.827735644637053)
(14, 0.7672955974842768)
(16, 0.7694753577106518)
(18, 0.7555673382820783)
(20, 0.7359263050153533)
(22, 0.74163179916318)
(24, 0.7404737384140061)
(26, 0.671858774662513)
(28, 0.7083758937691521)
(30, 0.6183476938672074)
(32, 0.6280825364871666)
      }  node[pos=0.95, anchor=east] {};
      \addplot [thick, smooth, cyan!70!black] coordinates {
        (4, 0.9561855670103094)
(6, 0.9353991596638657)
(8, 0.8873755893137769)
(10, 0.8526645768025078)
(12, 0.8488624052004332)
(14, 0.7929769392033545)
(16, 0.7927927927927928)
(18, 0.7709437963944856)
(20, 0.7656090071647902)
(22, 0.7552301255230126)
(24, 0.7595262615859939)
(26, 0.6812045690550362)
(28, 0.7196118488253318)
(30, 0.6401419158641662)
(32, 0.6346250629089079)
      }  node[pos=0.95, anchor=east] {};
%       \addplot [thick, smooth,dashed, green!70!black] coordinates {
%         (4, 0.9367500000000001)
% (6, 0.9304166666666666)
% (8, 0.9407499999999999)
% (10, 0.9364999999999999)
% (12, 0.93875)
% (14, 0.9326666666666666)
% (16, 0.9282499999999999)
% (18, 0.9334166666666666)
% (20, 0.93625)
% (22, 0.9330833333333334)
% (24, 0.9392499999999999)
% (26, 0.9364999999999999)
% (28, 0.9347500000000001)
% (30, 0.9374999999999999)
% (32, 0.934)
%       }  node[pos=0.95, anchor=east];
      \addplot [thick, smooth,dashed, magenta!70!black] coordinates {
       (4, 0.9584966156038477)
(6, 0.9296851115866706)
(8, 0.9062912714611597)
(10, 0.8750075478533905)
(12, 0.852693136349646)
(14, 0.8364319021377097)
(16, 0.8241402479995116)
(18, 0.8080917221922083)
(20, 0.7868971352592773)
(22, 0.7774437898280631)
(24, 0.7647897804787299)
(26, 0.7569819291246186)
(28, 0.7470807752497892)
(30, 0.7241913385357719)
(32, 0.7174983285722967)
      }  node[pos=0.95, anchor=east] {};
      \addplot [thick, smooth,dashed, cyan!70!black] coordinates {
        (4, 1.073684835530222)
(6, 1.0726994802812595)
(8, 1.0538480009604996)
(10, 1.0518084656723627)
(12, 1.0485072329947676)
(14, 1.0352449585175318)
(16, 1.0271211288253619)
(18, 1.0180082838105526)
(20, 1.0001813127039767)
(22, 0.9809426475892746)
(24, 0.9717846955227583)
(26, 0.9538253931002111)
(28, 0.9319248826291079)
(30, 0.9180913802347055)
(32, 0.8971008326748922)
      }  node[pos=0.95, anchor=east] {};
    \legend{CIP $q_0=4$ , CIL $q_0=4$, CIP $q_0=1.6$, CIL $q_0=1.6$}
   \end{axis}
\end{tikzpicture} 

% \end{document}

%% file: latex_figures/clf_recursive_acc.tex
% \documentclass[tikz]{standalone}

% \usepackage{pgfplots,siunitx,mathtools}

% \pgfplotsset{compat=newest}

% \begin{document}

\begin{tikzpicture}
  \begin{axis}[
    %   ymode=log,
    %   domain=1:33,
      ymax=0.81,
      enlargelimits=false,
      ylabel=Accuracy,
      xlabel= communication round $t$,
      grid=both,
      width=9cm,
      height=5.5cm,
      decoration={name=none},
      legend style={at={(0.99,0.1)},anchor=south east}
    ]
    \addplot [ thick, smooth, magenta!85!black] coordinates {
        (1, 0.5968332290649414)
(2, 0.6923333406448364)
(3, 0.7206665873527527)
(4, 0.7394999861717224)
(5, 0.7533333897590637)
(6, 0.7623332738876343)
(7, 0.7681666612625122)
(8, 0.7728332281112671)
(9, 0.7758334279060364)
(10, 0.7818334102630615)
(11, 0.7843334674835205)
(12, 0.7868334054946899)
(13, 0.7901667356491089)
(14, 0.7923334240913391)
(15, 0.7930000424385071)
(16, 0.7948334217071533)
(17, 0.7970000505447388)
(18, 0.7981666326522827)
    }node[pos=0.95, anchor=east] {};
    \addplot [ thick, smooth,cyan!85!black] coordinates {
        (1, 0.6634999513626099)
(2, 0.7098333239555359)
(3, 0.7315000891685486)
(4, 0.7445000410079956)
(5, 0.7566666603088379)
(6, 0.7656666040420532)
(7, 0.7693333625793457)
(8, 0.7745000123977661)
(9, 0.7791666984558105)
(10, 0.783833384513855)
(11, 0.783833384513855)
(12, 0.7878333926200867)
(13, 0.7903333306312561)
(14, 0.7916666269302368)
(15, 0.7948334813117981)
(16, 0.7970000505447388)
(17, 0.7971667051315308)
(18, 0.796833336353302)
    }node[pos=0.95, anchor=east] {};
    \addplot [ thick, smooth,dashed, magenta!85!black] coordinates {
      (1, 0.5243332386016846)
(2, 0.6269999742507935)
(3, 0.6583333611488342)
(4, 0.6763333082199097)
(5, 0.6913332939147949)
(6, 0.703000009059906)
(7, 0.7141666412353516)
(8, 0.7216666340827942)
(9, 0.7288333177566528)
(10, 0.737000048160553)
(11, 0.7409999966621399)
(12, 0.7459999322891235)
(13, 0.7486667037010193)
(14, 0.7519999742507935)
(15, 0.7549999356269836)
(16, 0.7571665644645691)
(17, 0.7616667747497559)
(18, 0.7645000219345093)
    }node[pos=0.95, anchor=east] {};
    \addplot [ thick, smooth, dashed,cyan!85!black] coordinates {
       (1, 0.6036666631698608)
(2, 0.6458333730697632)
(3, 0.6676665544509888)
(4, 0.6833332777023315)
(5, 0.6969999670982361)
(6, 0.7076665163040161)
(7, 0.7178332209587097)
(8, 0.7256666421890259)
(9, 0.7338333129882812)
(10, 0.7386667132377625)
(11, 0.7440000176429749)
(12, 0.7480000257492065)
(13, 0.7501667141914368)
(14, 0.7536666989326477)
(15, 0.7556667327880859)
(16, 0.7596666216850281)
(17, 0.763166606426239)
(18, 0.7666667103767395)
    }node[pos=0.95, anchor=east] {};
    \legend{CIP $M=4$,CIL $M=4$,CIP $M=16$,CIL $M=16$}
   \end{axis}
\end{tikzpicture}

% \end{document}

%% file: latex_figures/clf_recursive_var_mean.tex
% \documentclass[tikz]{standalone}

% \usepackage{pgfplots,siunitx,mathtools}

% \pgfplotsset{compat=newest}

% \begin{document}

\begin{tikzpicture}
  \begin{axis}[
      ymode=log,
    %   domain=1:33,
    %   ymax=0.81,
      enlargelimits=false,
      ylabel=Mean of weights Variance,
      xlabel= communication round $t$,
      grid=both,
      width=9cm,
      height=5.5cm,
      decoration={name=none},
      legend style={at={(0.01,0.01)},anchor=south west}
    ]
    \addplot [ thick, smooth,mark=*, magenta!85!black] coordinates {
        (1, 142.22894287109375)
(2, 11.671184539794922)
(3, 0.9253749847412109)
(4, 0.05756501108407974)
(5, 0.00287729618139565)
(6, 0.00011988511687377468)
(7, 4.281608198652975e-06)
(8, 1.3380024199705076e-07)
(9, 3.7166731914339834e-09)
(10, 9.291683394918593e-11)
(11, 2.111745891000827e-12)
(12, 4.399471170940354e-14)
(13, 8.460520831013799e-16)
(14, 1.5108073385199276e-17)
(15, 2.518012747854429e-19)
(16, 3.9343945146257616e-21)
(17, 5.785873172528371e-23)
(18, 8.035937700945326e-25)
    }node[pos=0.95, anchor=east] {};
    \addplot [ thick, smooth,cyan!85!black] coordinates {
        (1, 171.16781616210938)
(2, 14.378822326660156)
(3, 1.143081784248352)
(4, 0.07111905515193939)
(5, 0.0035548326559364796)
(6, 0.0001481153885833919)
(7, 5.28983127878746e-06)
(8, 1.6530722746210813e-07)
(9, 4.591867774905722e-09)
(10, 1.1479668604597038e-10)
(11, 2.609015867932607e-12)
(12, 5.43544904723324e-14)
(13, 1.0452786792185644e-15)
(14, 1.8665694481728593e-17)
(15, 3.1109482186416274e-19)
(16, 4.860856995524326e-21)
(17, 7.148318628467926e-23)
(18, 9.928220755572622e-25)
    }node[pos=0.95, anchor=east] {};
    \addplot [ thick, smooth,dashed,mark=*, magenta!85!black] coordinates {
      (1, 41.62830352783203)
(2, 1.1102595329284668)
(3, 0.023020297288894653)
(4, 0.00035965186543762684)
(5, 4.495638677326497e-06)
(6, 4.682956955548434e-08)
(7, 4.1812117657080705e-10)
(8, 3.26657169195943e-12)
(9, 2.2684526579754873e-14)
(10, 1.4177828053555611e-16)
(11, 8.055584215336212e-19)
(12, 4.1956163749241605e-21)
(13, 2.017123645537369e-23)
(14, 9.005016714932956e-26)
(15, 3.7520897361591645e-28)
(16, 1.4656599591476256e-30)
(17, 5.388456185935781e-33)
(18, 1.870991683396714e-35)
    }node[pos=0.95, anchor=east] {};
    \addplot [ thick, smooth, dashed,cyan!85!black] coordinates {
       (1, 41.37792205810547)
(2, 1.1710907220840454)
(3, 0.02431698516011238)
(4, 0.0003799234109465033)
(5, 4.749036634166259e-06)
(6, 4.946912213199539e-08)
(7, 4.4168860235949126e-10)
(8, 3.45069242277396e-12)
(9, 2.396314144836008e-14)
(10, 1.497696128764268e-16)
(11, 8.509638646215172e-19)
(12, 4.43210319230588e-21)
(13, 2.1308190123632525e-23)
(14, 9.512584612494128e-26)
(15, 3.963576825576056e-28)
(16, 1.548272385569743e-30)
(17, 5.6921793142752896e-33)
(18, 1.9764509754097433e-35)
    }node[pos=0.95, anchor=east] {};
    \legend{CIP $M=4$,CIL $M=4$,CIP $M=16$,CIL $M=16$}
   \end{axis}
\end{tikzpicture}

% \end{document}